%% file: main.tex
\documentclass{dukecei}

\usepackage{helvet}
\usepackage{courier}

\input{math_commands.tex}
\usepackage[utf8]{inputenc}
\usepackage[T1]{fontenc}

\usepackage{newtxtext,newtxmath}

\usepackage{amsmath}
\usepackage{amsfonts}
\usepackage{amssymb}
\usepackage{amsthm}
\usepackage{mathtools}
\usepackage{bbm}
\usepackage{bm}
\usepackage{nicefrac}

\usepackage{graphicx}
\usepackage{wrapfig}
\usepackage{float}
\usepackage{adjustbox}

\usepackage[table]{xcolor}
\usepackage{array}
\usepackage{booktabs}
\usepackage{multirow}
\usepackage{makecell}
\usepackage{tabularx}
\usepackage{longtable}

\newcolumntype{C}[1]{>{\centering\arraybackslash}p{#1}}
\newcolumntype{L}[1]{>{\raggedright\arraybackslash}p{#1}}

\definecolor{basisblue}{rgb}{0.93,0.96,1.0}
\definecolor{basisbg}{RGB}{229,227,247}
\definecolor{hlblue}{HTML}{D9F2F2}

\definecolor{pressureblue}{HTML}{EEF4FA}
\definecolor{pressurebluehead}{HTML}{DCE9F5}

\definecolor{paperblue}{HTML}{1A5B8C}
\definecolor{steelblue}{HTML}{2E75B6}
\definecolor{darkteal}{HTML}{006666}

\definecolor{tocblue}{RGB}{0,0,120}
\definecolor{accentpink}{HTML}{F00080}

\newcommand{\redhead}[1]{%
  \textcolor{accentpink}{\textbf{\textsc{#1}}}%
}

\usepackage{algorithm}
\usepackage{algorithmic}

\usepackage{textcomp}
\usepackage[inline]{enumitem}
\usepackage{ragged2e}
\usepackage{setspace}

\usepackage{caption}
\usepackage{subcaption}

\usepackage{tocloft}
\usepackage{titletoc}

\usepackage{mdframed}
\usepackage[skins,breakable]{tcolorbox}

\surroundwithmdframed[
  linewidth=0.5pt,
  linecolor=gray!20,
  tikzsetting={dash pattern=on 3pt off 3pt},
  backgroundcolor=gray!10,
  roundcorner=4pt,
  innerleftmargin=6pt,
  innerrightmargin=6pt,
  innertopmargin=8pt,
  innerbottommargin=6pt
]{theorem}

\surroundwithmdframed[
  linewidth=0.5pt,
  linecolor=gray!20,
  tikzsetting={dash pattern=on 3pt off 3pt},
  backgroundcolor=gray!10,
  roundcorner=4pt,
  innerleftmargin=6pt,
  innerrightmargin=6pt,
  innertopmargin=8pt,
  innerbottommargin=6pt
]{proposition}

\surroundwithmdframed[
  linewidth=0.5pt,
  linecolor=gray!20,
  tikzsetting={dash pattern=on 3pt off 3pt},
  backgroundcolor=gray!10,
  roundcorner=4pt,
  innerleftmargin=6pt,
  innerrightmargin=6pt,
  innertopmargin=8pt,
  innerbottommargin=6pt
]{criteria}

\surroundwithmdframed[
  linewidth=0.5pt,
  linecolor=gray!20,
  tikzsetting={dash pattern=on 3pt off 3pt},
  backgroundcolor=gray!10,
  roundcorner=4pt,
  innerleftmargin=6pt,
  innerrightmargin=6pt,
  innertopmargin=8pt,
  innerbottommargin=6pt
]{theoremx}

\usepackage{microtype}

\usepackage{url}
\usepackage{hyperref}

\hypersetup{
    colorlinks=true,
    linkcolor=tocblue,
    citecolor=tocblue,
    urlcolor=tocblue
}

\title{
The Functional Structure of Post-Compression Recovery in Low-Rank LLMs
}

\author{%
Zishan Shao$^{1,\dagger}$,
Liang Tian$^{1}$,
Georgiy Zemlevskiy$^{1}$,
Kangning Cui$^{4}$,
Lixun Zhang$^{1}$,
Yixiao Wang$^{1}$,
Ting Jiang$^{2}$,
Jinhee Kim$^{2}$,
Yixuan Chen$^{5}$,
Rui-Feng Wang$^{3}$,
Fan Yang$^{4}$,
Hai Li$^{1}$,
Yiran Chen$^{1}$
\\[0.65em]
\normalfont\small
$^{1}$Duke University
\quad
$^{2}$Carnegie Mellon University
\\
$^{3}$University of Florida
\quad
$^{4}$Wake Forest University
\quad
$^{5}$University of Oxford
}

\begin{document}

\maketitle
\thispagestyle{firstpagestyle}

\begin{abstract}

Different low-rank compression methods can produce compressed LLMs that respond differently to the same post-compression recovery procedure, and relative advantages observed between methods at the compression endpoint may shrink, grow, or even reverse after recovery. We ask whether this recovery heterogeneity reflects functional structure beyond scalar loss evolution, and how that structure evolves throughout recovery. Our results establish that this heterogeneity reflects a reproducible compression-induced functional structure, which we formalize as \textbf{recovery pressure}. To characterize this structure consistently throughout recovery, we develop a standardized functional characterization
within each backbone that is applicable across heterogeneous low-rank methods. The primary backward characterization reveals reproducible module-wise structure across independent probes, while a complementary forward-only characterization recovers related structure without loss or backpropagation. We further find that recovery pressure measured at the endpoint is associated with subsequent recovery response; during recovery, its module-wise structure is reorganized non-uniformly, and localized updates induce distributed responses beyond directly updated modules. Further evidence indicates that tracking this evolving structure provides a complementary functional view of recovery progress alongside scalar loss.


\end{abstract}










\vspace{-0.2cm}
\section{Introduction}

\input{text/introduction}

\section{Related Work}

\input{text/relatedwork}

\section{Recovery Pressure: A Functional Structure Induced by Low-Rank Compression}
\label{sec:recoverability}

\input{text/methodology}




\section{Recovery Pressure at the Compression Endpoint}
\label{sec:endpoint-pressure}

\input{text/findings}

\section{Recovery Reshapes Pressure Structure}

\input{text/recovery_dynamics}

\section{Recovery Pressure as a Complementary View of Recovery Dynamics}
\label{sec:beyond-loss}

\input{text/findings/finding_4}

\section{Conclusion}

\input{text/conclusion}

\clearpage
\bibliographystyle{abbrvnat}
\bibliography{main}



\clearpage
\beginsupplement


\onecolumn
\appendix
{\hypersetup{linkcolor=black}
\tableofcontents
}

\vspace{0.5cm}
\begin{center}
    \LARGE {\bf Appendix}
\end{center}

\appendix
\section{Appendix}

\subsection{Additional Design} \label{app:additional_design}

\input{text/appendix_design}

\subsection{Additional Experiments} \label{app:additional_exp}

\input{text/appendix_experiment}

\end{document}

%% file: math_commands.tex
\usepackage{amsmath,amsfonts,bm}

\def\eqref#1{equation~\ref{#1}}

\def\1{\bm{1}}

\DeclareMathAlphabet{\mathsfit}{\encodingdefault}{\sfdefault}{m}{sl}
\SetMathAlphabet{\mathsfit}{bold}{\encodingdefault}{\sfdefault}{bx}{n}

%% file: text/introduction.tex
\begin{wrapfigure}{r}{0.36\textwidth}
  \centering
  \vspace{-24pt} 
  \includegraphics[width=\linewidth]{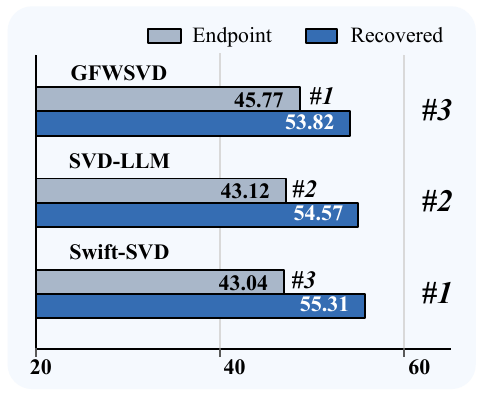}
  \caption{
\textbf{Endpoint rankings need not be preserved after recovery.} Avg7 MCQ Performance~(\ref{sec:exp-setup}) on Qwen3-8B-Base, $r=0.6$.
}
  \label{fig:intro-motivation}
  \vspace{-15pt} 
\end{wrapfigure}




Low-rank factorization has become an increasingly popular approach to
compressing LLMs~\citep{
wang2025svdllm,wang2025svdllmv2,yuan2023asvd,wang2025dobi,
abbasi2026zssvd,gao2026dfsvd,hu2026saes,lin2024modegpt,
chekalina2025gfwsvd,wang2025basis,ding2025dipsvd,gao2025gfsvd,
li2025adasvd,mi2025drank,sinha2026aasvd}.
However, compressed models produced by different low-rank methods can respond
very differently to the same post-compression recovery procedure. Relative
advantages observed at the compression endpoint may
\textit{\textbf{shrink, grow, or even reverse after recovery}}
(Fig.~\ref{fig:intro-motivation}). This recovery heterogeneity shows that
endpoint performance specifies where a compressed model starts, but does not
fully describe how it evolves under subsequent recovery.


These divergent recovery responses raise a more fundamental question: do they simply reflect differences in scalar performance evolution, or do they reveal functional structure in compressed models? This motivates our central question:

\newpage

\begin{tcolorbox}[
  enhanced,
  frame hidden,
  borderline west={2pt}{0pt}{black!70},
  borderline east={2pt}{0pt}{black!70},
  colback=gray!10,
  boxrule=0pt,
  arc=3pt,
  halign=center,
  top=6pt,
  bottom=6pt,
  before skip=8pt,
  after skip=8pt
]
\textbf{Does post-compression recovery exhibit functional structure beyond scalar loss evolution, and how is this structure organized and reshaped over recovery?}
\end{tcolorbox}


Answering this question requires separating recovery-dependent responses from
confounds introduced by the recovery procedure itself. Two issues are
particularly important. First, recovery trajectories depend on the choice of training setups (e.g., data, optimizer, etc.). Second, unconstrained recovery can alter the coordinates selected by compression, making it difficult to distinguish functional evolution within those coordinates from changes to the coordinates themselves. We therefore compare compressed models under the same recovery intervention and use fixed-coordinate recovery to reduce these recovery-side confounds.


With these confounds controlled, we need a functional characterization that is
comparable across heterogeneous low-rank methods. Because their internal
parameterizations can differ substantially by methods, we compare matched perturbation
responses at architecture-aligned module outputs relative to the dense
reference. This provides a standardized functional interface within each
backbone that applies across low-rank methods and can be evaluated consistently
throughout recovery.


Under this common interface, our primary characterization reveals reproducible
module-wise functional structure across independent probes. We further test
whether related structure remains visible under a qualitatively different
forward-only characterization at the compression endpoint, without task loss
or backpropagation. Across compression methods and model backbones, the two
characterizations recover related functional structure. We formalize this
reproducible compression-induced structure as
\textbf{\emph{recovery pressure}}.



Because the same characterization can be applied throughout recovery, we can directly study how recovery pressure evolves. We find that its module-wise structure is not uniformly reduced: relative module contributions are systematically reorganized over time, while localized recovery at a single MLP module induces distributed changes across non-updated attention and MLP blocks. Recovery therefore exhibits trackable functional evolution beyond its scalar loss trajectory. Finally, we examine whether this functional structure is relevant to recovery itself. At the compression endpoint, recovery pressure measured by our characterization is associated with subsequent recovery response. These findings reveal post-compression recovery in low-rank compressed LLMs as a \emph{functionally structured process}; tracking this evolving structure further provides a complementary functional view of recovery progress alongside scalar loss.







Our contributions are: 
\begin{itemize}

\item We discover that post-compression recovery exhibits \emph{reproducible functional structure beyond scalar loss evolution}, which we formalize as \textbf{\emph{recovery pressure}}.

\item We develop a standardized characterization procedure of recovery pressure within each backbone and applicable across diverse low-rank methods, enabling recovery pressure to be tracked throughout post-compression recovery.

\item We systematically study the evolution of \emph{recovery pressure} and find that its module-wise structure is non-uniformly reorganized over time and offers a complementary functional characterization of recovery progress alongside scalar loss.




\end{itemize}

%% file: text/relatedwork.tex
\paragraph{Low-Rank Compression for LLMs.}
Recent low-rank LLM compression methods improve factorized models through
activation-aware transformations, reconstruction- or loss-aware truncation,
gradient-informed objectives, adaptive rank allocation, and cross-layer
sharing~\citep{
yuan2023asvd,wang2025svdllmv2,wang2025dobi,abbasi2026zssvd,
gao2026dfsvd,gao2025gfsvd,hu2026saes,ding2025dipsvd,
li2025adasvd,mi2025drank,chekalina2025gfwsvd,lin2024modegpt,
wang2025basis,sinha2026aasvd}.
SVD-LLM~\citep{wang2025svdllm} further introduces sequential low-rank updates
for post-compression recovery, while Low-Rank Prehab~\citep{qin2025low}
prepares models before SVD to improve both compression and subsequent
recovery. These works primarily seek stronger compressed representations and
better endpoint quality. Our question is different: \emph{whether post-compression recovery exhibits functional structure beyond scalar loss evolution, and how that structure changes throughout recovery.}

\vspace{-0.3cm}
\paragraph{Post-Compression Recovery.}
Post-compression adaptation is widely used to restore model quality after
pruning and other forms of compression, including LoRA-style recovery and
structure-aware compensation~\citep{
hu2022lora,ma2023llm,chen2023lorashear,liu2026laco,feng2026restoring}.
Recovery has also been incorporated into compression decisions:
PRACTISE~\citep{wang2023practise} estimates recoverability for block pruning,
while Relative Repairability~\citep{zhan2026repairability} measures residual
damage after lightweight repair for sparsity allocation. In these settings,
recovery primarily serves to improve adaptation or guide compression choices.
We instead use a controlled recovery intervention to study the functional
organization induced by compression and how it is reorganized by subsequent
updates. In particular, fixed-coordinate recovery preserves the
compression-selected coordinates, reducing changes to those coordinates as a
confound in this comparison.


\vspace{-0.3cm}
\paragraph{Sensitivity Measures and Functional Diagnostics.} Gradient- and sensitivity-based quantities have long been used to characterize parameter or component importance, including Taylor criteria, SNIP, and loss sensitivity to gated Transformer components~\citep{molchanov2017pruning,lee2018snip,michel2019sixteen}. More recently, TraceNAS compares parameter-gradient traces between pretrained and structurally pruned models as a zero-shot proxy for functional inheritance and post-pruning recovery potential~\citep{malettira2026tracenas}. We do not claim novelty for the derivative primitive itself. Rather, we ask whether post-compression recovery exhibits reproducible functional structure that persists across probe sets, is visible through complementary functional characterizations, and is systematically reorganized by recovery. We use \textbf{recovery pressure} to name this empirical structure rather than any particular sensitivity metric or candidate-ranking proxy. Task-conditioned backward sensitivity provides our primary temporal characterization, while forward-only perturbation responses provide complementary evidence at the compression endpoint.

%% file: text/methodology.tex
\subsection{From Recovery Heterogeneity to Functional Structure}
\label{sec:pressure-object}

\begin{wrapfigure}{r}{0.44\textwidth}
  \centering
  \vspace{-20pt} 
  \includegraphics[width=\linewidth]{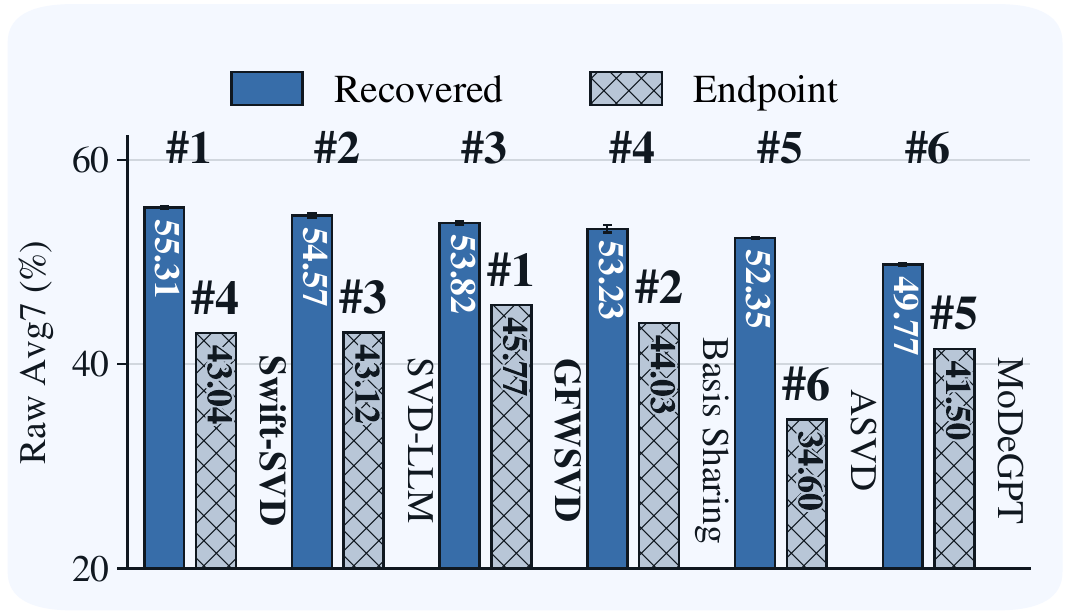}
  \caption{
    \textbf{Recovery reshapes comparative performance.}
    Endpoint (hatched) and recovered (solid) Avg7~(\ref{sec:exp-setup}) across representative settings, with endpoint ranks annotated above each method.
  }
  \label{fig:ranking-instability}
  \vspace{-27pt} 
\end{wrapfigure}





Post-compression recovery can substantially reshape the relative performance
of low-rank compressed models. Pairwise advantages observed at the compression
endpoint may shrink, grow, or even reverse after recovery (Fig.~\ref{fig:ranking-instability} and Appendix~\ref{app:full_endpoint_recovery}). 

\paragraph{Observation: Endpoint comparative advantages are not preserved under recovery.} \emph{The relative performance margins among compressed models can change substantially after post-compression recovery, in some cases enough to reverse their endpoint ordering.}


This recovery heterogeneity motivates us to ask whether the heterogeneous recovery responses exhibit systematic functional structure beyond scalar loss evolution. We formalize this structure as \textbf{\emph{recovery pressure}} below.

\vspace{2pt}
\noindent\colorbox{pressureblue}{%
\parbox{\dimexpr\columnwidth-2\fboxsep\relax}{%
\begin{definition}[\redhead{Recovery pressure}]
\label{def:recovery-pressure}
We use \textbf{\emph{recovery pressure}} to denote the \emph{compression-induced functional structure underlying heterogeneity} in post-compression recovery responses.
\end{definition}%
}}
\vspace{2pt}

\subsection{Disentangling Recovery from Intervention Effects and Coordinate Drift}
\label{sec:controlled-recovery}

Two difficulties complicate the accurate characterization of post-compression recovery. First, the observed trajectory depends on the recovery intervention design (e.g., data, optimizer, schedule, etc.), which confounds model evaluations unless carefully controlled. Second, unconstrained recovery can alter the compression-selected coordinates selected during compression, making it hard to distinguish whether updates are reorganizing the compression-selected coordinates or replacing it entirely (Appendix~\ref{app:fixed_coordinate}).

We address the first difficulty by declaring a shared finite-budget recovery intervention. Let $\mathcal I$ denote this intervention and $d_0$ denote the dense reference of the compressed endpoint $c_0$. The intervention generates $c_t=\mathcal R_t(c_0;\mathcal I),$ with $0\le t\le T.$ We compare compressed models under the same $\mathcal I$, so differences in their recovery responses are evaluated under same optimization conditions.





\begin{definition}[Finite-Budget Recoverability]
\label{def:finite_budget_recovery}
For compressed endpoint $c_0$ under intervention $\mathcal I$, let $D_t
=
\mathcal L_{\mathrm{val}}(c_t)
-
\mathcal L_{\mathrm{val}}(d_0)$ denote the remaining signed loss gap to the dense reference. We define
the \textbf{recovery response} at step $t$ as $H_t
=
D_0-D_t
=
\mathcal L_{\mathrm{val}}(c_0)
-
\mathcal L_{\mathrm{val}}(c_t),$ and characterize the finite-budget recoverability of $c_0$ under
$\mathcal I$ by the recovery-response trajectory $\mathcal H_T(c_0;\mathcal I)=\{H_t\}_{t=0}^{T}.$
\end{definition}

We use $H_T$ for fixed-budget comparisons and the full trajectory $\mathcal H_T$ for trajectory-level analyses. Accordingly, recoverability in this work is intervention-conditioned. 


Fixing $\mathcal I$ controls the optimization conditions, but does not ensure
preservation of the compression-selected coordinates. If the low-rank factors are updated with no parameter constraints, the compression-selected coordinates can
rotate or drift, confounding structural evolution within the endpoint with replacement of its defining coordinates. We therefore
restrict recovery to a fixed-coordinate parameterization to address this second
challenge.

Specifically, for a compressed operator $\widehat W_i = W_{2,i} W_{1,i}$ with
$W_{2,i} \in \mathbb R^{d_{\mathrm{out}} \times k_i}$ and
$W_{1,i} \in \mathbb R^{k_i \times d_{\mathrm{in}}}$, the compression-selected
outer factors define the fixed coordinate class
\begin{equation}
\mathcal F_i
=
\left\{
W_{2,i} C W_{1,i}
\;\middle|\;
C\in\mathbb R^{k_i\times k_i}
\right\}.
\label{eq:coordinate-class}
\end{equation}

\begin{definition}[Fixed-Coordinate Recovery]
\label{def:fixed-coordinate}
Fixed-coordinate recovery freezes $W_{2,i}$ and $W_{1,i}$ and adapts only an
internal middle map
$C_i=I_{k_i}+A_iB_i$, where
$A_i\in\mathbb R^{k_i\times r_a}$,
$B_i\in\mathbb R^{r_a\times k_i}$, and $r_a\ll k_i$.
The recovered operator is
\begin{equation}
\widetilde W_i
=
W_{2,i}C_iW_{1,i}
=
W_{2,i}(I_{k_i}+A_iB_i)W_{1,i}.
\label{eq:fixed-coordinate}
\end{equation}
\end{definition}

\begin{proposition}[Coordinate Preservation]
\label{prop:coordinate-preservation}
Any operator $\widetilde W_i$ produced by
Definition~\ref{def:fixed-coordinate} remains in the compression-selected
coordinate class $\widetilde W_i\in\mathcal F_i$. Consequently, $\operatorname{col}(\widetilde W_i)
\subseteq
\operatorname{col}(W_{2,i}),$ and $
\operatorname{row}(\widetilde W_i)
\subseteq
\operatorname{row}(W_{1,i}).$
If $C_i$ is nonsingular and $W_{2,i},W_{1,i}$ have full column and row rank,
respectively, both containments hold with equality.
\end{proposition}

Thus, recovery modifies interactions within compression-selected coordinates
without introducing new outer directions. We adopt this controlled parameterization to remove coordinate drift as a confound when studying the evolution of recovery pressure.

\subsection{Characterizing Recovery Pressure}
\label{sec:functional-characterization}


Within the recovery setting above, we now characterize
\textbf{recovery pressure} throughout the recovery trajectory. Because recovery pressure represents the functional structure induced across heterogeneous low-rank compression methods, whose internal parameterizations
can differ substantially, a direct comparison in parameter space is not
well-defined for our purpose. \textbf{\emph{Corresponding Transformer module outputs, in contrast, form model architecture-aligned functional boundaries shared across these low-rank methods.}}

This observation motivates a central design choice in our characterization: at corresponding module outputs, we apply matched perturbations to the compressed or recovering model and its dense reference, and compare the resulting loss changes. We subtract the corresponding dense-model change to isolate the compression-induced functional discrepancy.

A second design choice concerns how these shared functional boundaries should be perturbed. We use multiplicative channel-wise gates because they provide a relative, identity-centered perturbation interface whose coordinates are explicitly localized to individual output channels. This makes the response decomposable by module and allows perturbations to be restricted to selected modules. Let $\mathcal Q$ be a fixed probe set and let $z_m(x)\in\mathbb R^{L_x\times d_m}$ denote the output of module $m\in\mathcal M$ before residual addition. We perturb $z_m(x)\mapsto g_m\odot z_m(x),$ with each channel gate shared across tokens and all module gates collected into $g$.

\begin{figure}
    \centering
    \includegraphics[width=0.99\linewidth]{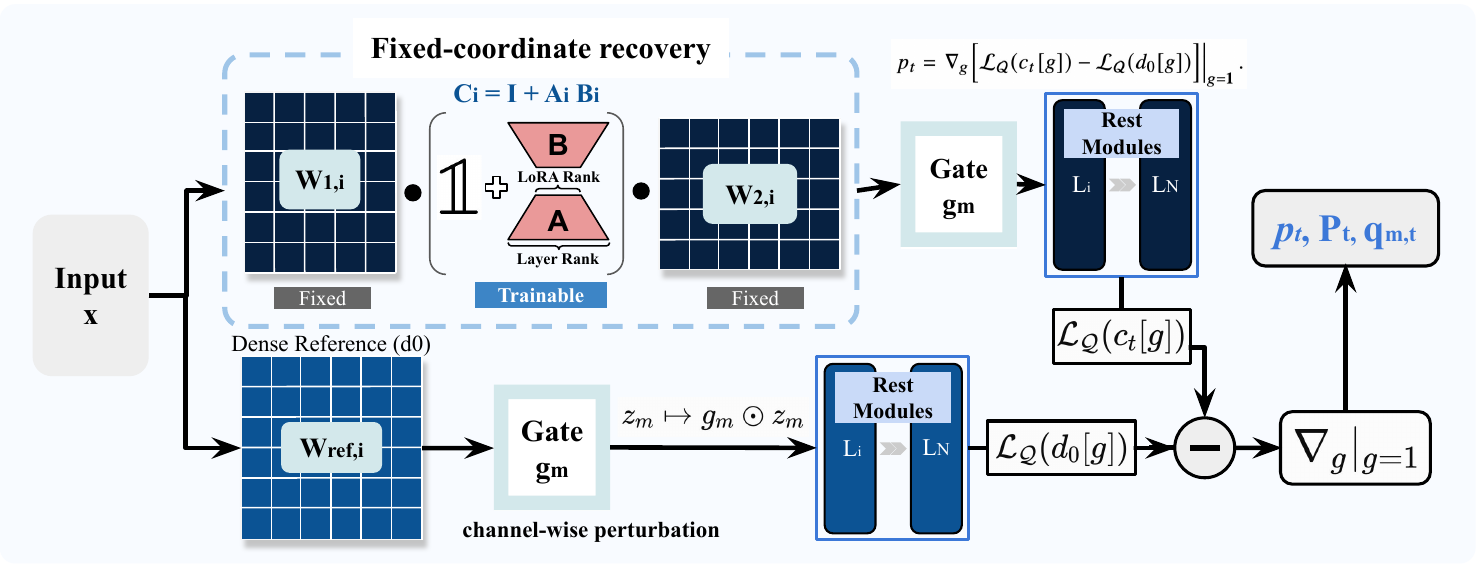}
    \caption{
    \textbf{Fixed-coordinate recovery and backward characterization of recovery pressure.}
    Recovery updates the trainable middle map while preserving compression-selected coordinates.
    Matched channel-wise perturbations at corresponding outputs of $c_t$ and $d_0$ define the dense-referenced response $p_t$, summarized by $P_t$ and $q_{m,t}$.
    }
    \label{fig:pipeline}
\end{figure}

\begin{definition}[Task-conditioned backward characterization]
\label{def:backward-pressure}
For recovery point $c_t$ and dense reference $d_0$, define the
compressed--dense loss-response contrast
\begin{equation}
p_t
=
\left.
\nabla_g
\Big[
\mathcal L_{\mathcal Q}(c_t[g])
-
\mathcal L_{\mathcal Q}(d_0[g])
\Big]
\right|_{g=\mathbf 1}.
\label{eq:pressure}
\end{equation}
We summarize this characterization by its overall magnitude
$P_t=\|p_t\|_2$ and its module-wise structure $q_{m,t}
=
\frac{\|p_{m,t}\|_2^2}{\|p_t\|_2^2},$ with $ p_t=\operatorname{concat}_{m\in\mathcal M}p_{m,t},$
$P_t>0.$ At the compression endpoint, these quantities are denoted $P_0$ and
$q_{m,0}$.
\end{definition}

Because the channel-wise gates define relative perturbations around the unperturbed network, the resulting characterization admits a direct
first-order functional interpretation.

\begin{proposition}[Local sensitivity interpretation]
\label{prop:local-sensitivity}
Let
$\Delta_t(g)=
\mathcal L_{\mathcal Q}(c_t[g])-
\mathcal L_{\mathcal Q}(d_0[g])$.
If $\Delta_t$ is differentiable at $g=\mathbf 1$, then for any perturbation
direction $v$,
\[
D\Delta_t(\mathbf 1)[v]=p_t^\top v,
\qquad
P_t
=
\max_{\|v\|_2\le 1}
\big|D\Delta_t(\mathbf 1)[v]\big|.
\]
Thus $P_t$ is the largest first-order compressed--dense difference in local
loss response under a unit Euclidean gate perturbation.
\end{proposition}

In this way, the gate construction was connected directly to the functional information measured by $p_t$: each gate coordinate probes the effect of a relative rescaling at a specific output channel, while $P_t$ summarizes the strongest first-order discrepancy over such perturbation directions. Because our intervention design is fixed, the same characterization can be evaluated consistently along the recovery trajectory. The corresponding channel-level derivative and the proof of Proposition~\ref{prop:local-sensitivity} are given in Appendix~\ref{app:pressure}.

\subsection{Experimental Setup}
\label{sec:exp-setup}

Unless stated otherwise, we evaluate six low-rank compression methods on Llama-3.1-8B~\citep{grattafiori2024llama3} and Qwen3-8B-Base~\citep{yang2025qwen3} at retention ratios $r\in\{0.8,0.6\}$. Recovery uses the common fixed-coordinate intervention described above, with C4~\citep{raffel2020c4} training data and a budget of 384 updates. We report zero-shot Avg7 over BoolQ~\citep{clark2019boolq}, PIQA~\citep{bisk2020piqa}, HellaSwag~\citep{zellers2019hellaswag}, WinoGrande~\citep{sakaguchi2020winogrande}, ARC-Easy~\citep{clark2018think}, ARC-Challenge~\citep{clark2018think}, and OpenBookQA~\citep{mihaylov2018openbookqa}, and use held-out language-model loss for recovery-response analyses. The primary characterization is computed relative to the dense reference using the module-output perturbation interface in Section~\ref{sec:functional-characterization}. Detailed endpoint analyses use an independently sampled 128-window probe set (Q128), while recovery response is measured on V128, a separate 128-window validation set. Recovery results use three seeds unless a fixed-seed trajectory is explicitly stated; additional section-specific interventions and diagnostics are introduced where needed.

%% file: text/findings.tex
\begin{figure}
    \centering
    \includegraphics[width=0.99\linewidth]{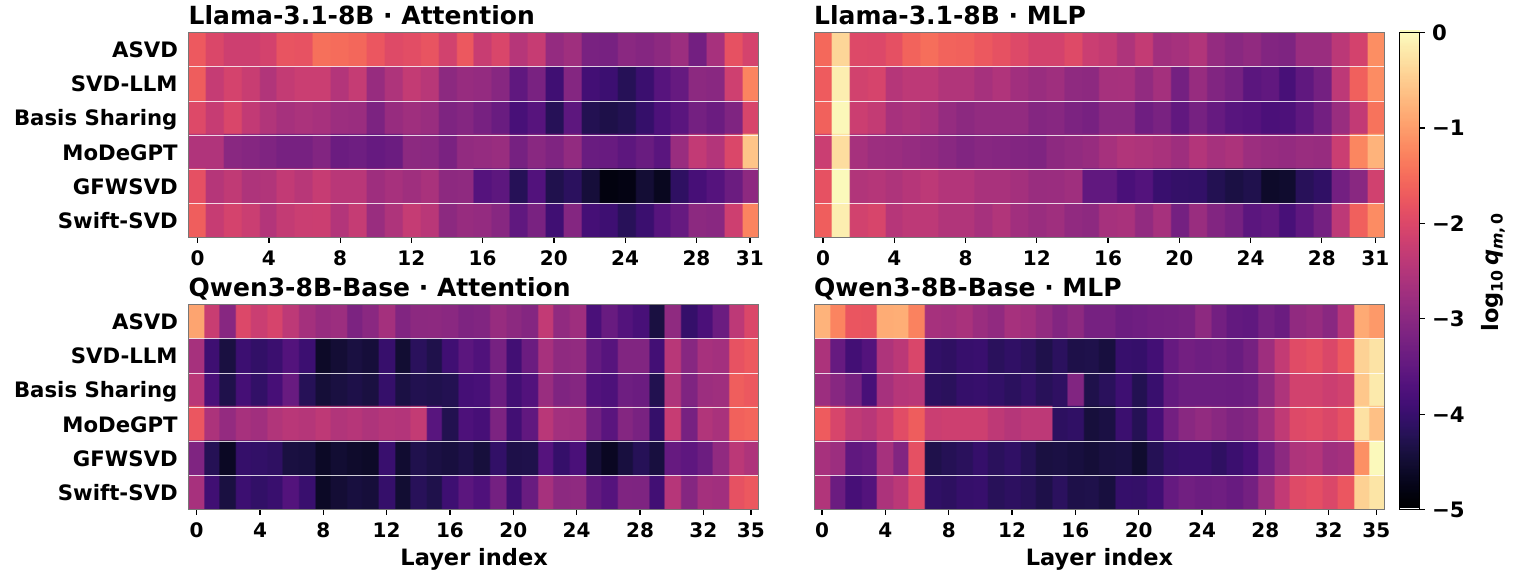}
    \caption{\textbf{Module-wise organization of endpoint recovery pressure.}
    Colors show $\log_{10}q_{m,0}$ for $r=0.8$ endpoints on the independent
    Q128 probe, normalized over attention and MLP modules.}
    \label{fig:pressure_organization}
\end{figure}

We now apply the primary characterization introduced in Section~\ref{sec:functional-characterization} at the compression endpoint and ask what it reveals about the effects of low-rank compression beyond scalar endpoint quality. We begin by examining the response patterns that emerge across compressed endpoints and whether they persist under independently sampled probes.

\paragraph{\textbf{\textit{The backward characterization reveals reproducible module-wise structure.}}} Figure~\ref{fig:pressure_organization} reveals a markedly non-uniform module-wise distribution of the endpoint response. Across both backbones, the response is concentrated in a subset of modules and is predominantly MLP-centered, while the dominant layers and their relative contributions vary across compression methods. The characterization therefore exposes both a recurring cross-endpoint pattern and substantial endpoint-specific structure that is not represented by scalar endpoint quality alone.

These patterns are largely reproducible under independent probes. Comparing Q64 with the independently sampled Q128 probe, ten of twelve $r=0.8$ endpoints retain strong agreement in module ranking (Spearman $\rho=0.968$--$0.996$), with MoDeGPT providing the main exception. The broad ordering of responsive modules is more stable than the exact allocation of profile mass, motivating the use of Q128 for the detailed endpoint analysis. Full results are reported in Appendix~\ref{app:source_profile_reproducibility}.

\begin{wraptable}{r}{0.50\textwidth}
\centering
\vspace{-10pt} 
\caption{
\textbf{Forward-only grounding.}
Statistics summarize six compression methods within each
backbone--retention setting. $\rho_{F,B}$ is reported as median [range];
$\rho_F^{\rm split}$ and Top-8 are medians.
}
\label{tab:forward-grounding}
\small
\setlength{\tabcolsep}{3.0pt}
\renewcommand{\arraystretch}{1.08}
\begin{tabular}{ccccc}
\toprule
Backbone & $r$ & $\rho_{F,B}$ median [range] & $\rho_F^{\mathrm{split}}$ & Top-8 \\
\midrule
\multirow{2}{*}{\shortstack{Llama-3.1\\8B}}
& 0.8 & .613 [.542, .910] & .987 & 5.5/8 \\
& 0.6 & .611 [.293, .832] & .983 & 5.0/8 \\
\midrule
\multirow{2}{*}{\shortstack{Qwen3-8B\\Base}}
& 0.8 & .377 [.042, .701] & .603 & 3.0/8 \\
& 0.6 & .471 [.295, .594] & .629 & 4.5/8 \\
\bottomrule
\end{tabular}
\vspace{-10pt} 
\end{wraptable}

\paragraph{\textbf{\textit{Forward-only perturbations recover related endpoint structure.}}} Because our primary characterization is task-conditioned and loss-based, we next ask whether the observed endpoint patterns remain visible under a characterization that does not use task loss or backpropagation. We therefore construct a complementary forward-only characterization measuring output responses directly.

For the forward-only characterization, we measure at each module the compressed--dense difference in the output response induced by matched local perturbations, and normalize these responses across modules to obtain a forward profile $q_m^{\rm F}$; the full construction is given in Appendix~\ref{app:forward_backward_alignment}. We compare this profile with the backward profile $q_{m,0}$ using Spearman rank correlation $\rho_{F,B}$. We additionally report $\rho_F^{\rm split}$, the agreement between forward profiles computed on two disjoint 16-window halves, and Top-8 overlap between the highest-ranked forward and backward modules.

\begin{wraptable}{r}{0.36\textwidth}
\centering
\caption{
\textbf{Endpoint magnitude and recovery.}
Spearman $\rho(P_0,H_T)$ across six low-rank methods.
}
\label{tab:pressure-recovery}
\small
\setlength{\tabcolsep}{5.0pt}
\renewcommand{\arraystretch}{1.08}
\begin{tabular}{ccc}
\toprule
Backbone & $r$ & $\rho(P_0,H_T)$ \\
\midrule
\multirow{2}{*}{\shortstack{Llama-3.1\\8B}}
& 0.8 & 1.000 \\
& 0.6 & 0.886 \\
\midrule
\multirow{2}{*}{\shortstack{Qwen3-8B\\Base}}
& 0.8 & 0.771 \\
& 0.6 & 0.829 \\
\bottomrule
\end{tabular}
\vspace{-8pt}
\end{wraptable}

Table~\ref{tab:forward-grounding} shows positive forward--backward module-rank alignment across all 24 compressed endpoints. The correspondence is consistently stronger on Llama, where it is also highly reproducible across independent input splits; Qwen exhibits weaker and more method-dependent alignment. The dominant forward responses nevertheless overlap systematically with those emphasized by the primary backward characterization. The two characterizations therefore expose related, rather than identical, endpoint structure. We retain the backward characterization as the primary task-conditioned view used to track recovery, while the forward-only result provides complementary grounding that related functional structure is visible without loss or backpropagation. Additional module-level comparisons are reported in Appendix~\ref{app:forward_backward_alignment}.


\begin{table*}[t]
\centering

\caption{
\textbf{Recovery-aware comparison of low-rank compressed endpoints.}
$P_0$ is measured on the Q128 probe before recovery.
Endpoint and Recovered report zero-shot Avg7 (\%) before and after
384 C4 updates under the shared fixed-coordinate intervention;
recovered scores are averaged over three seeds.
Bold marks the highest Endpoint or Recovered Avg7 within each
backbone--retention setting. Full recovery gains and seed variation are
reported in the Appendix~\ref{app:pressure_stopping_seeds}.
}
\label{tab:recovery_aware_comparison}

\scriptsize
\setlength{\tabcolsep}{2.8pt}
\renewcommand{\arraystretch}{1.10}

\resizebox{\textwidth}{!}{%
\begin{tabular}{
l
>{\columncolor{pressureblue}}c cc |
>{\columncolor{pressureblue}}c cc |
>{\columncolor{pressureblue}}c cc |
>{\columncolor{pressureblue}}c cc
}
\toprule
&
\multicolumn{3}{c|}{\textbf{Llama-3.1-8B} $r=.8$} &
\multicolumn{3}{c|}{\textbf{Llama-3.1-8B} $r=.6$} &
\multicolumn{3}{c|}{\textbf{Qwen3-8B-Base} $r=.8$} &
\multicolumn{3}{c}{\textbf{Qwen3-8B-Base} $r=.6$} \\
\cmidrule(lr){2-4}
\cmidrule(lr){5-7}
\cmidrule(lr){8-10}
\cmidrule(lr){11-13}

Method
& \cellcolor{pressurebluehead}\textbf{$P_0$}
& \textbf{Endpoint}
& \textbf{Recovered}
& \cellcolor{pressurebluehead}\textbf{$P_0$}
& \textbf{Endpoint}
& \textbf{Recovered}
& \cellcolor{pressurebluehead}\textbf{$P_0$}
& \textbf{Endpoint}
& \textbf{Recovered}
& \cellcolor{pressurebluehead}\textbf{$P_0$}
& \textbf{Endpoint}
& \textbf{Recovered} \\
\midrule

ASVD
& 1.314 & 33.55 & 55.14
& 1.285 & 32.55 & 44.86
& \textcolor{paperblue}{\textbf{0.241}} & \textbf{59.61} & \textbf{62.40}
& 1.464 & 34.60 & 52.35 \\

SVD-LLM
& 0.192 & 51.94 & 58.80
& 0.567 & 36.81 & 50.92
& 0.336 & 55.93 & 60.57
& 0.430 & 43.12 & 54.57 \\

Basis Sharing
& 0.374 & 51.17 & 57.97
& 1.425 & 33.67 & 48.47
& 0.450 & 53.97 & 60.42
& 0.680 & 44.03 & 53.23 \\

MoDeGPT
& \textcolor{paperblue}{\textbf{0.074}} & \textbf{61.10} & \textbf{62.45}
& \textcolor{paperblue}{\textbf{0.159}} & \textbf{45.94} & \textbf{53.80}
& 0.421 & 54.48 & 59.35
& 0.550 & 41.50 & 49.77 \\

GFWSVD
& 0.943 & 48.34 & 57.25
& 1.701 & 32.98 & 48.78
& 0.544 & 55.25 & 60.64
& 1.288 & \textbf{45.77} & 53.82 \\

Swift-SVD
& 0.194 & 51.55 & 58.91
& 0.568 & 36.39 & 51.16
& 0.344 & 56.10 & 61.09
& \textcolor{paperblue}{\textbf{0.434}} & 43.04 & \textbf{55.31} \\

\bottomrule
\end{tabular}%
}
\end{table*}

\paragraph{\textbf{\textit{Endpoint pressure is associated with subsequent recovery response.}}}
We next ask whether the magnitude of the endpoint characterization is related to the subsequent recovery response under the shared intervention. Following the design in Sections~\ref{sec:controlled-recovery} and~\ref{sec:exp-setup}, we measure the Spearman association between endpoint $P_0$ and $H_T$ across the six compression methods within each backbone--retention setting. Across all four settings, $P_0$ is positively rank-associated with $H_T$ (Table~\ref{tab:pressure-recovery}).

This analysis has a deliberately limited role: it connects the functional
structure characterized at the compression endpoint to subsequent recovery, rather than establishing an endpoint predictor beyond scalar loss. Indeed, $P_0$ is closely associated with endpoint loss
(Appendix~\ref{app:pressure_loss_relation}), and we do not claim incremental predictive value over endpoint loss for $H_T$. Table~\ref{tab:recovery_aware_comparison} places this association alongside downstream performance before and after recovery. Larger endpoint pressure can coincide with a larger recovery response without implying stronger final downstream performance, since endpoints that begin from different quality levels may still finish at different recovered performance levels.

Thus, the $P_0$--$H_T$ association provides evidence that the endpoint functional characterization is related to the subsequent recovery process, but not as a ranking criterion for recovered models. Whether the evolving recovery-pressure characterization provides information complementary to scalar loss is evaluated separately in Section~\ref{sec:beyond-loss}.

%% file: text/recovery_dynamics.tex
\begin{figure}
    \centering
    \includegraphics[width=0.99\linewidth]{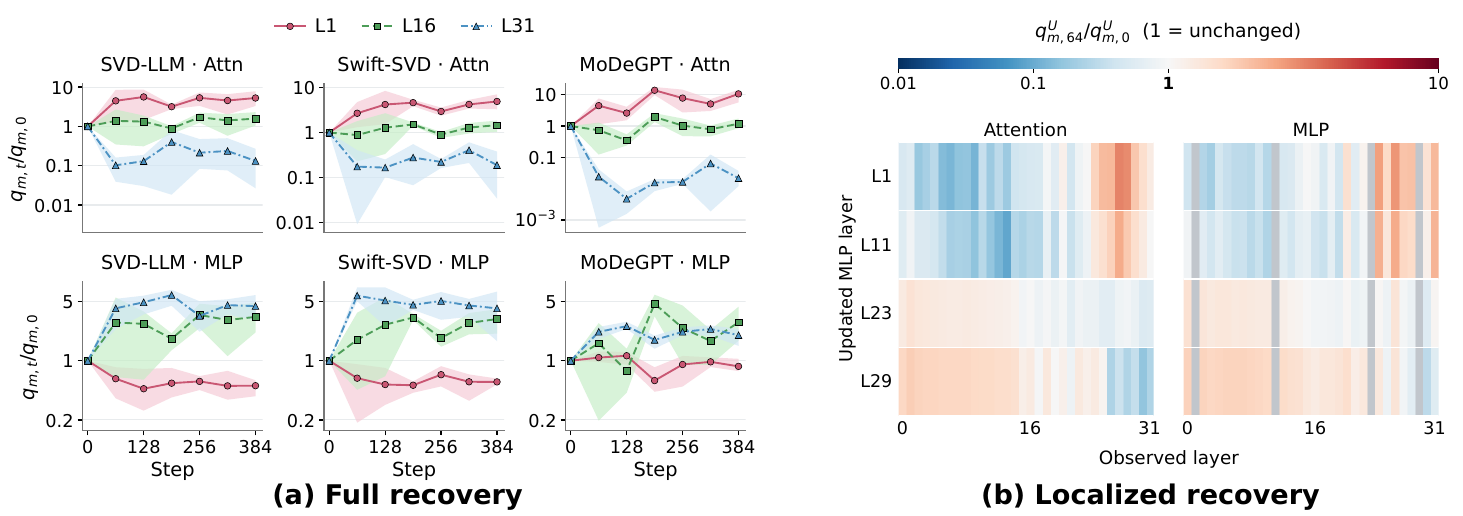}
\caption{\textbf{Nonuniform pressure reorganization under full and localized recovery.} \textbf{(a)} Baseline-normalized module shares on Llama-3.1-8B ($r=0.8$, C4; mean $\pm$ range across 3 seeds). Flat line at 1.0 indicates uniform rescaling. \textbf{(b)} Share changes across unupdated modules after 64-step single-MLP recovery (SVD-LLM; 1.0 = no change).}
    \label{fig:s6_module_trajectories}
\end{figure}

We next examine how recovery pressure evolves over time---whether recovery
mainly changes its overall magnitude or also reorganizes how it is distributed
across modules.

\paragraph{\textit{Recovery changes module-wise structure beyond changes in overall magnitude.}}

If recovery changed all module responses proportionally, then
\[
\|p_{m,t}\|_2^2=\alpha_t\|p_{m,0}\|_2^2
\]
for every module, leaving the normalized shares unchanged:
$q_{m,t}=q_{m,0}$.

Figure~\ref{fig:s6_module_trajectories}(a) instead shows heterogeneous module-wise evolution. Across methods, attention modules in earlier layers and MLP modules in later layers tend to gain relative pressure share, while attention modules in later layers tend to lose share; several trajectories are also non-monotonic over recovery. Notably, this redistribution can occur even as the corresponding absolute responses decrease: for SVD-LLM and Swift-SVD, all tracked module responses are lower in absolute magnitude at the end, while their relative shares continue to shift. These observations show that recovery redistributes the functional response across modules rather than changing all module contributions proportionally. 


This reorganization is not limited to the trajectories visualized in
Fig.~\ref{fig:s6_module_trajectories}(a): across the complete $r=0.8$ cohort of six methods and two backbones, all 12 source-to-final profile changes exceed their
paired-document error references (Appendix~\ref{app:s6_module_dynamics}), confirming this redistribution is ubiquitous.

\paragraph{\textbf{\textit{Local updates reshape recovery pressure across the network.}}} We next ask whether the module-wise reorganization observed above is confined to the modules being updated. We restrict recovery to a single MLP down-projection at layer 1, 11, 23, or 29 for 64 updates, freeze all remaining parameters, and measure the resulting changes only at modules that are not directly updated.

Figure~\ref{fig:s6_module_trajectories}(b) shows that these localized
interventions nevertheless produce distributed changes across both attention
and MLP modules. The response also depends strongly on the update location:
earlier and later interventions induce distinct, and in several regions
opposing, changes in module-wise pressure share. Thus, changing a single
recovery site can redistribute recovery pressure well beyond the updated module. 

%% file: text/findings/finding_4.tex
The preceding results establish that recovery pressure is a reproducible functional structure that evolves throughout post-compression recovery. We next ask whether tracking this structure offers a complementary view of recovery dynamics beyond the scalar loss trajectory.

Directly isolating the complementary value of recovery pressure is challenging because different recovery trajectories do not traverse identical loss values at the same step. We therefore frame recovery stopping as a controlled diagnostic setting: using identical saved recovery trajectories and checkpoints, we compare a validation-loss-only criterion (\textsc{Loss}) against a pressure-informed criterion (\textsc{Loss+P}) that incorporates the evolving pressure characterization alongside loss observations.

Under this setup, optimization, checkpoints, and available loss metrics remain strictly fixed, making recovery pressure the sole additional variable. \textsc{Loss+P} augments \textsc{Loss} by requiring both pressure magnitude and its module-wise profile to stabilize. Both rules are calibrated under identical intervention and loss-preservation constraints, and we further sweep across nearby loss operating points to verify that the observed differences are not specific to a single threshold (Appendix~\ref{app:pressure-stopping}).


\paragraph{\textbf{\textit{Recovery pressure provides a complementary view of recovery progress alongside scalar loss.}}} Adding pressure changes the mean selected recovery point in all six settings (Table~\ref{tab:downstream-main}). In five settings, \textsc{Loss+P} stops earlier, reducing recovery updates by $8.57$--$45.56\%$ relative to \textsc{Loss}; on Qwen3-8B-Base at $r=.8$, it instead continues slightly longer. Despite these changes, the mean Avg7 difference between the two rules remains within $0.223$ percentage points in every setting.

Importantly, pressure does not simply impose a more aggressive stopping rule. Across recovery trajectories, it can either advance or delay stopping, and this behavior persists across independent recovery seeds and nearby loss operating points (Appendix~\ref{app:pressure-stopping}). Because the two rules observe the same recovery checkpoints and loss values, these differences show that the evolving pressure characterization can lead to different recovery decisions even under the same observed loss trajectory. This complementarity does not imply uniform dominance over loss: for example, at Llama-3.1-8B $r=0.4$, \textsc{Loss+P} reaches similar Avg7 with fewer updates but a higher C4 PPL. Rather, pressure provides an complementary functional view of recovery progress alongside scalar loss.

\begin{table*}[t]
\centering
\caption{
\textbf{Controlled evaluation of pressure-informed stopping.}
Loss and Loss+P are evaluated on identical recovery trajectories and checkpoints, differing only by incorporating pressure dynamics. Each row averages six compression methods (Full uses 384 updates at $r=0.8/0.6$ and 1024 at $r=0.4$). Saving is relative to Loss; C4 PPL uses 256 validation windows of 256 tokens.
}
\label{tab:downstream-main}
\small
\setlength{\tabcolsep}{3.2pt} 
\renewcommand{\arraystretch}{1.12}

\begin{tabular}{ll c r >{\columncolor{pressureblue}}c >{\columncolor{pressureblue}}c c c c c c}
\toprule
&& \multicolumn{3}{c}{\textbf{Recovery updates}}
 & \multicolumn{3}{c}{\textbf{Avg7} $\uparrow$}
 & \multicolumn{3}{c}{\textbf{C4 PPL} $\downarrow$} \\
\cmidrule(lr){3-5} \cmidrule(lr){6-8} \cmidrule(lr){9-11}
Backbone & $r$
& Loss 
& \multicolumn{1}{>{\columncolor{pressurebluehead}}c}{\textbf{Loss+P}} 
& \multicolumn{1}{>{\columncolor{pressurebluehead}}c}{\textbf{\textbf{\shortstack{Update Red-\\uction(\%)}}}}
& Full & Loss & Loss+P
& Full & Loss & Loss+P \\
\midrule
\multirow{3}{*}{\shortstack[l]{Llama-3.1\\8B}} 
 & 0.8 & 341.3 & 256.0 & \textbf{25.00} & 58.351 & 58.309 & 58.331 & 16.70 & 16.70 & 16.83 \\
 & 0.6 & 362.7 & 288.0 & \textbf{20.59} & 49.803 & 49.825 & 49.782 & 24.66 & 24.69 & 25.10 \\
 & 0.4 & 960.0 & 522.7 & \textbf{45.56} & 40.640 & 40.520 & 40.516 & 44.22 & 44.52 & 48.49 \\
\midrule
\multirow{3}{*}{\shortstack[l]{Qwen3-8B\\Base}} 
 & 0.8 & 330.7 & 341.3 & -3.23          & 61.029 & 60.794 & 61.017 & 16.57 & 16.58 & 16.59 \\
 & 0.6 & 373.3 & 341.3 & \textbf{8.57}  & 53.194 & 53.234 & 53.156 & 22.01 & 22.07 & 22.12 \\
 & 0.4 & 906.7 & 725.3 & \textbf{20.00} & 44.725 & 44.711 & 44.753 & 32.54 & 32.88 & 33.41 \\
\bottomrule
\end{tabular}
\end{table*}

\paragraph{\textbf{\textit{Implications for recovery-aware evaluation.}}}
Beyond revealing the functional structure of post-compression recovery, our characterization provides a standardized within-backbone interface that can be applied consistently from the compression endpoint throughout subsequent recovery across heterogeneous low-rank methods. This connection matters because recovery can substantially reshape the relative performance of compressed models, making endpoint-only evaluation incomplete when recovery is part of the compression pipeline. At the endpoint, recovery pressure characterizes the functional structure left by compression, and, during recovery, tracks how that structure is reorganized and provides a complementary functional characterization of recovery progress alongside scalar loss. This can be cruicial when designing post-compression recovery awared low-rank methods. 


%% file: text/conclusion.tex
We introduced recovery pressure to characterize the compression-induced functional structure underlying heterogeneous post-compression recovery responses in low-rank LLMs. Using a common characterization interface across low-rank methods, we show that this structure is reproducible across independent probes and can be tracked consistently from the compression endpoint throughout recovery. Recovery does not simply reduce this structure uniformly: its module-wise organization is systematically reshaped over time, and localized updates induce distributed responses beyond the modules directly updated. We further find that endpoint recovery pressure is associated with subsequent recovery response, while its evolution provides a complementary functional view of recovery progress alongside scalar loss. These results show that post-compression recovery is not merely a process of scalar loss evolution, but a structured functional reorganization of the compressed model. More broadly, recovery pressure provides a common lens connecting the compression endpoint with subsequent recovery, motivating recoverability as an explicit dimension of low-rank compression evaluation and design.

%% file: text/appendix_design.tex
\subsubsection{Recovery-Pressure Measurement}
\label{app:pressure}



\paragraph{Measurement.}
Recovery pressure denotes the compression-induced functional structure studied in the main text. We characterize it through matched perturbations at corresponding module outputs of a compressed or recovering model $c_t$ and its fixed dense parent $d_0$.

For either model $M\in\{c_t,d_0\}$, let $z_m(x)\in\mathbb{R}^{L_x\times d_m}$ denote the output of module $m$ before residual addition. We apply a channel-wise multiplicative gate $g_m\in\mathbb{R}^{d_m}$, with $\widetilde z_{m,\tau,j}(x) = g_{m,j} \cdot z_{m,\tau,j}(x),$ where $\tau\in\{1,\ldots,L_x\}$ indexes tokens and $j\in\{1,\ldots,d_m\}$ indexes output channels. The same gate coordinate is shared across tokens, and $g_m=\mathbf 1$ recovers the unperturbed module output. For an input $x$,
\[
\left.\nabla_g\ell(M[g];x)\right|_{g=\mathbf 1} = \left.
\frac{\partial \ell(M[g];x)}
     {\partial g_{m,j}}
\right|_{g=\mathbf 1}
=
\sum_{\tau=1}^{L_x}
z_{m,\tau,j}(x)
\frac{\partial \ell(M;x)}
     {\partial z_{m,\tau,j}(x)} .
\]

Hence each gate derivative is the first-order loss response to an infinitesimal relative rescaling of one output channel. The gate interface is identity-centered and decomposable across modules and channels, so these local responses can be compared at the same architecture-aligned functional boundaries across compressed, recovering, and dense models. No finite perturbation magnitude needs to be chosen.

The resulting geometry is defined with respect to this fixed channel-wise gate interface. In particular, the Euclidean norm used below is part of the measurement convention and is not claimed to be invariant under arbitrary changes of output-channel basis.

For a fixed probe set $\mathcal Q$, define $a_{\mathcal Q}(M)
=
\frac{1}{|\mathcal Q|}
\sum_{x\in\mathcal Q}
\left.\nabla_g\ell(M[g];x)\right|_{g=\mathbf 1}.$
The backward characterization in Section~\ref{sec:functional-characterization} is then
\[
p_t
=
a_{\mathcal Q}(c_t)-a_{\mathcal Q}(d_0)
=
\left.
\nabla_g
\Big[
\mathcal L_{\mathcal Q}(c_t[g])
-
\mathcal L_{\mathcal Q}(d_0[g])
\Big]
\right|_{g=\mathbf 1}.
\]
Subtracting the dense response removes components shared with the dense parent:
$p_t$ vanishes when $c_t$ and $d_0$ have identical local responses even if
their individual gate gradients are nonzero. We summarize the contrast by
$P_t=\|p_t\|_2$ and
$q_{m,t}=\|p_{m,t}\|_2^2/\|p_t\|_2^2$.

The probe set, dense reference, module boundaries, channel ordering, and
response definition are fixed throughout recovery. In particular,
\[
p_t-p_0
=
a_{\mathcal Q}(c_t)-a_{\mathcal Q}(c_0),
\]
so temporal changes are measured under the same convention.

For Proposition~\ref{prop:local-sensitivity}, let
$\Delta_t(g)=
\mathcal L_{\mathcal Q}(c_t[g])-
\mathcal L_{\mathcal Q}(d_0[g])$.
Differentiability at $g=\mathbf 1$ gives
$D\Delta_t(\mathbf 1)[v]=p_t^\top v$.
By Cauchy--Schwarz,
$|p_t^\top v|\le \|p_t\|_2$ for $\|v\|_2\le1$, with equality for
$v=p_t/\|p_t\|_2$ when $p_t\neq0$. Hence
\[
P_t
=
\max_{\|v\|_2\le1}
\big|D\Delta_t(\mathbf 1)[v]\big|.
\]

\paragraph{Overall magnitude and module-wise organization.}
We summarize the overall magnitude by $P_t=\|p_t\|_2$. Writing $p_t=\operatorname{concat}_{m\in\mathcal M}p_{m,t}$, the module-wise pressure share is $q_{m,t}=\|p_{m,t}\|_2^2/\|p_t\|_2^2$ for $P_t>0$, with $\sum_m q_{m,t}=1$. Thus $P_t$ captures overall magnitude, while $q_{m,t}$ describes the relative organization across modules.

\begin{proposition}[Local sensitivity interpretation]
\label{prop:pressure-local}
If $\Delta_t$ is differentiable at $g=\mathbf 1$, then for any perturbation
direction $v$,
\[
D\Delta_t(\mathbf 1)[v]=p_t^\top v,
\qquad
P_t
=
\max_{\|v\|_2\leq 1}
\left|D\Delta_t(\mathbf 1)[v]\right|.
\]
Thus $P_t$ is the largest first-order compressed--dense loss-response
difference under a unit Euclidean gate perturbation.
\end{proposition}

\begin{proof}
Because $\Delta_t$ is differentiable at $g=\mathbf 1$, its directional
derivative along any direction $v$ is given by $D\Delta_t(\mathbf 1)[v]
=
\left.
\nabla_g\Delta_t(g)
\right|_{g=\mathbf 1}^{\!\top} v
=
p_t^\top v,$ where the last equality follows from the definition
$p_t=\left.\nabla_g\Delta_t(g)\right|_{g=\mathbf 1}$. For any $v$ satisfying $\|v\|_2\leq 1$, Cauchy--Schwarz gives
\[
\left|D\Delta_t(\mathbf 1)[v]\right|
=
|p_t^\top v|
\leq
\|p_t\|_2\|v\|_2
\leq
\|p_t\|_2
=
P_t.
\]
Hence,
\[
\max_{\|v\|_2\leq 1}
\left|D\Delta_t(\mathbf 1)[v]\right|
\leq P_t.
\]

If $p_t\neq 0$, choosing
$v=p_t/\|p_t\|_2$ gives $\|v\|_2=1$ and
\[
\left|D\Delta_t(\mathbf 1)[v]\right|
=
\left|
p_t^\top\frac{p_t}{\|p_t\|_2}
\right|
=
\|p_t\|_2
=
P_t,
\]
so the upper bound is attained. If $p_t=0$, both sides are zero. Therefore, $P_t
=
\max_{\|v\|_2\leq 1}
\left|D\Delta_t(\mathbf 1)[v]\right|.$
\end{proof}

Signed response differences are averaged before taking the norm because the
target is the response contrast of the mean probe objective. Averaging
per-window norms would instead retain opposing response directions as positive
contributions.

\paragraph{Comparability along recovery.}
The probe set, dense reference, module-output locations, channel ordering,
perturbation interface, and loss reduction remain fixed throughout recovery;
only $c_t$ changes. In particular,
$p_t-p_0=a_{\mathcal Q}(c_t)-a_{\mathcal Q}(c_0)$, so the dense contribution
cancels in temporal comparisons. The same measurement can therefore be applied
throughout recovery on a common scale, while both $P_t$ and $q_{m,t}$ may
change non-monotonically.

%% file: text/appendix_experiment.tex
\input{tables/full_pressure_recovery}

\subsubsection{Full Endpoint and Recovery Results}
\label{app:full_endpoint_recovery}

Table~\ref{tab:main_recoverability} reports the complete six-method comparison for both backbones at $r=0.8$ and $r=0.6$. Endpoint Avg7 is measured once per frozen Source; recovered Avg7 is the mean and sample SD over three paired recovery seeds. Ranks are recomputed within each backbone--retention setting from unrounded scores. The $P_0$ column is the fixed Q0 endpoint-pressure readout and is not an additional recovery score.

\subsubsection{Recovery Protocol}
\label{app:recovery_protocol}
The shared recovery intervention takes 384 optimizer updates, each with an effective batch of eight 2,048-token windows (microbatch one and eight accumulation steps). It consumes 3,072 distinct windows once, totaling 6,291,456 input positions. Only native BASIS adapters are trainable; the compressed Source weights and selected outer coordinates remain frozen. We use AdamW with a constant learning rate of $5\times10^{-4}$, $\beta=(0.9,0.999)$, $\epsilon=10^{-8}$, zero weight decay, and global gradient clipping at 1. Adapter weights and optimizer moments use BF16, with BF16 autocast. The input/down adapter factor is initialized from $\mathcal N(0,0.02^2)$ and the output/up factor at zero. Recovery seeds are 7711, 7712, and 7713; within a seed and backbone, all methods use the same C4 windows and order. Table~\ref{tab:recovery_protocol_capacity} reports the realized adapter capacity. The budget rule targets 11 million trainable parameters within 5\% for each Source. Here $r_a$ is the adapter middle-map rank at a module, and the implementation scales its low-rank update by $\alpha/r_a=2$.

\begin{table*}[t]
\centering
\small
\setlength{\tabcolsep}{5pt}
\caption{\textbf{Native BASIS capacity for the 384-update C4 recovery.}
Each cell gives adapter rank $r_a$ / exact trainable parameter count for one
frozen compression endpoint. Counts are identical across the three recovery
seeds. Shared Basis Sharing adapters are counted
once per shared basis.}
\label{tab:recovery_protocol_capacity}
\begin{tabular*}{\textwidth}{@{\extracolsep{\fill}}lrrrr@{}}
\toprule
& \multicolumn{2}{c}{Llama-3.1-8B} & \multicolumn{2}{c}{Qwen3-8B-Base} \\
\cmidrule(lr){2-3}\cmidrule(l){4-5}
Method & $r=0.8$ & $r=0.6$ & $r=0.8$ & $r=0.6$ \\
\midrule
ASVD          & 19 / 10,725,120 & 21 / 11,137,728 & 31 / 10,963,770 & 20 / 11,261,960 \\
SVD-LLM       & 14 / 10,958,080 & 19 / 11,151,936 & 13 / 11,191,752 & 17 / 10,975,608 \\
Basis Sharing & 18 / 10,892,160 & 25 / 11,083,200 & 17 / 11,122,488 & 22 / 10,794,960 \\
MoDeGPT       &  4 / 10,878,704 & 5--6 / 10,994,530 & 4 / 10,826,464 & 5--6 / 10,990,644 \\
GFWSVD        & 14 / 10,958,080 & 19 / 11,151,936 & 13 / 11,191,752 & 17 / 10,975,608 \\
Swift-SVD     & 14 / 10,958,080 & 19 / 11,151,936 & 13 / 11,191,752 & 17 / 10,975,608 \\
\bottomrule
\end{tabular*}
\end{table*}

\subsubsection{Pressure-Informed Stopping Protocol}
\label{app:pressure-stopping}

Prior to applying either rule, a fixed 64-window validation set ($V_{64}$) is split into 32 monitoring windows ($M_{32}$) and 32 held-out evaluation windows ($E_{32}$). Both rules evaluate loss on $M_{32}$ to determine stopping, reserving $E_{32}$ strictly for post-hoc checkpoint evaluation. \textsc{Loss+P} additionally measures recovery pressure on a separate, non-overlapping 16-window probe ($Q_{16} \cap V_{64} = \emptyset$). All windows contain 256 tokens. $M_{32}$ loss and $Q_{16}$ pressure are recorded at initialization and every 64 recovery updates, reusing identical windows throughout. 

At checkpoint $j$, \textsc{Loss} triggers stopping when validation loss improvement $L_{j-1}-L_j$ remains at or below its threshold for $k$ consecutive 64-update intervals. \textsc{Loss+P} applies a parallel loss condition with the NLL threshold, while further requiring both the relative pressure-magnitude change $\frac{|P_j-P_{j-1}|}{P_0}$ and module-profile shift $\frac{1}{2}\sum_m |q_{m,j}-q_{m,j-1}|$ to stay at or below their respective thresholds over those same $k$ intervals. If a rule does not trigger early, the full-budget checkpoint is returned.

\subsubsection{Relationship Between Endpoint Magnitude and Scalar Endpoint Quality}
\label{app:pressure_loss_relation}

As expected, endpoint magnitude is closely aligned with scalar endpoint
quality. Within each backbone and validation set, the dense-reference loss is
fixed, so $\rho(P_0,D_0)=\rho(P_0,L_0)$ for
$D_0=L_0-\mathcal{L}_{\mathrm{val}}(d_0)$. The observed associations are
consistent with the intuitive tendency for more degraded endpoints to exhibit
both larger functional discrepancies and greater room for subsequent recovery.

\begin{table*}[t]
\centering
\footnotesize
\setlength{\tabcolsep}{4pt}
\caption{
\textbf{Endpoint magnitude, scalar endpoint quality, and subsequent recovery.} Panel A reports within-setting Spearman correlations across the six compression methods. $L_0$ is held-out endpoint NLL and $H_{384}=L_0-L_{384}$ is the mean held-out loss reduction over three recovery seeds. Q0 denotes the original 16-window pressure probe; Q128/V128 independently remeasures all endpoints using separate pressure and validation sets. Panel B reports a retrospective method-held-out prediction audit over all
24 Q128/V128 endpoints.
}
\label{tab:pressure_loss_relation}

\begin{tabular}{llcrrr}
\toprule
\multicolumn{6}{l}{\textbf{A. Within-setting endpoint rank associations}}\\
\midrule
Probe & Backbone & $r$
& $\rho(P_0,L_0)$
& $\rho(P_0,H_{384})$
& $\rho(L_0,H_{384})$ \\
\midrule
Q0 & Llama & 0.8 & 1.000 & 1.000 & 1.000 \\
Q0 & Llama & 0.6 & 0.886 & 0.886 & 1.000 \\
Q0 & Qwen  & 0.8 & 0.771 & 0.771 & 1.000 \\
Q0 & Qwen  & 0.6 & 0.943 & 0.829 & 0.943 \\
\midrule
Q128/V128 & Llama & 0.8 & 1.000 & 1.000 & 1.000 \\
Q128/V128 & Qwen  & 0.8 & 0.771 & 0.771 & 1.000 \\
\bottomrule
\end{tabular}

\par\smallskip

\begin{tabular}{lrrrr}
\toprule
\multicolumn{5}{l}{\textbf{B. Method-held-out prediction: mean absolute error}}\\
\midrule
Target & Loss & Loss+$\log P_0$ & Loss shape & Loss shape+$\log P_0$ \\
\midrule
$D_T$ (nats/token) & 0.096849 & 0.079426 & 0.069537 & 0.067217 \\
$F_T$ (fraction)   & 0.111136 & 0.108753 & 0.050936 & 0.060390 \\
\bottomrule
\end{tabular}
\end{table*}

The method-held-out audit in Panel B is consistent with this picture: adding $\log P_0$ improves prediction of the remaining loss gap $D_T$ under both loss baselines, while gains for the recovered fraction $F_T$ depend on the loss parameterization. In the main analysis, the $P_0$--$H_T$ association serves to connect the endpoint functional characterization to subsequent recovery, whereas complementary value alongside loss is evaluated directly through the recovery-decision experiments in Section~\ref{sec:beyond-loss}.

\subsubsection{Coordinate-Drift Diagnostics for Factor-Update Recovery}
\label{app:fixed_coordinate}

\paragraph{Factor updates shift coordinates.}
We quantify coordinate integrity via the \textbf{collapse score} (negative log-volume ratio). Figure~\ref{fig:app_real_factor_geometry_uv} shows that unconstrained tuning (U/V/UV-LoRA) induces significant \textbf{coordinate drift}, shifting ECDFs toward higher collapse values. BASIS avoids this by keeping the outer factors fixed and preserving the selected SVD coordinates by construction.

\begin{figure}
    \centering
    \includegraphics[width=0.99\linewidth]{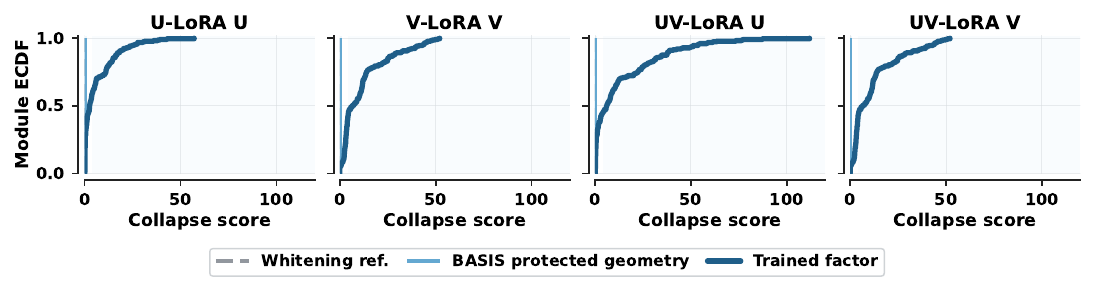}
    \caption{\textbf{Factor-update recovery shifts selected coordinates.}
    Collapse-score ECDFs on Llama-3.1-8B show that directly updated factors in U-LoRA, V-LoRA, and UV-LoRA move away from the whitening reference, while BASIS keeps the protected outer factors fixed.}
    \label{fig:app_real_factor_geometry_uv}
\end{figure}

\input{tables/main_numeric_geometry_shift}

\paragraph{Coordinate-shift diagnostic.}
We project learned endpoints back into the Source-defined fixed-coordinate subspace (Table~\ref{tab:geometry_shift_main}). While U/V-LoRA and JFT achieve competitive accuracy in learned coordinates, their projected endpoints show large Avg7 drops and severe PPL spikes; JFT drops from 50.26 to 40.85 Avg7. This suggests factor-update endpoints are sensitive to coordinates outside the Source coordinate system, whereas BASIS remains invariant by construction.





\subsubsection{Endpoint-Profile Reproducibility and Probe-Size Sensitivity}
\label{app:source_profile_reproducibility}

\paragraph{Reproducibility across independent probes.}
We compare the full module-share profiles $\{q_{m,0}\}_{m\in\mathcal M}$
obtained from the original 64-window aggregate (Q64) and an independently
sampled 128-window probe (Q128) at $r=0.8$, holding the endpoint, dense
reference, perturbation interface, and measurement definition fixed.

\begin{table}[t]
\centering
\small
\setlength{\tabcolsep}{7pt}
\renewcommand{\arraystretch}{1.08}
\caption{\textbf{Reproducibility of endpoint module organization.}
Spearman rank correlation and cosine similarity compare module-share profiles
from Q64 and independent Q128 probes over all attention and MLP modules.}
\label{tab:source_profile_reproducibility}
\begin{tabular}{lrrrr}
\toprule
& \multicolumn{2}{c}{Llama-3.1-8B}
& \multicolumn{2}{c}{Qwen3-8B-Base} \\
\cmidrule(lr){2-3}\cmidrule(lr){4-5}
Method & Spearman & Cosine & Spearman & Cosine \\
\midrule
ASVD          & 0.996 & 0.963    & 0.993 & 0.996 \\
SVD-LLM       & 0.968 & 0.983    & 0.978 & 0.997 \\
Basis Sharing & 0.988 & 0.996    & 0.991 & 0.999 \\
MoDeGPT       & 0.358 & 0.337    & 0.493 & 0.997 \\
GFWSVD        & 0.991 & $>0.999$ & 0.988 & $>0.999$ \\
Swift-SVD     & 0.970 & 0.987    & 0.984 & 0.997 \\
\bottomrule
\end{tabular}
\end{table}

Ten of twelve endpoints show strong rank agreement
(Spearman $0.968$--$0.996$). MoDeGPT is the main exception; Qwen MoDeGPT
retains high cosine similarity despite weaker rank agreement, whereas Llama
MoDeGPT is less stable under both measures. Thus, module organization is
reproducible for most tested endpoints, with method-dependent exceptions.

\paragraph{Sensitivity to probe size.}
Because Table~\ref{tab:main_recoverability} computes $P_0$ using the original
16-window probe (Q0), we additionally compare its module-share profile with
Q128.

\begin{table}[t]
\centering
\small
\setlength{\tabcolsep}{7pt}
\renewcommand{\arraystretch}{1.08}
\caption{\textbf{Module-profile sensitivity to the original probe size.}
Module-share agreement between the original 16-window Q0 and independent Q128
probes at $r=0.8$.}
\label{tab:source_profile_reproducibility_q0}
\begin{tabular}{lrrrr}
\toprule
& \multicolumn{2}{c}{Llama-3.1-8B}
& \multicolumn{2}{c}{Qwen3-8B-Base} \\
\cmidrule(lr){2-3}\cmidrule(lr){4-5}
Method & Spearman & Cosine & Spearman & Cosine \\
\midrule
ASVD          & 0.958 & 0.477    & 0.948 & 0.986 \\
SVD-LLM       & 0.913 & 0.739    & 0.936 & 0.999 \\
Basis Sharing & 0.960 & 0.905    & 0.948 & 0.998 \\
MoDeGPT       & 0.201 & 0.338    & 0.539 & 0.986 \\
GFWSVD        & 0.968 & 0.983    & 0.907 & $>0.999$ \\
Swift-SVD     & 0.912 & 0.756    & 0.940 & 0.999 \\
\bottomrule
\end{tabular}
\end{table}

For the ten non-MoDeGPT endpoints, Q0 preserves the broad module ordering
(Spearman $0.907$--$0.968$), while cosine agreement is more variable,
especially on Llama-3.1-8B. This suggests that coarse module ordering is less
sensitive to probe size than the exact allocation of profile mass, motivating
Q128 for the detailed organization analysis.

\subsubsection{Forward--Backward Characterization Alignment}
\label{app:forward_backward_alignment}

\paragraph{Forward characterization.}
For the complementary forward characterization, we apply matched local
perturbations to each compressed endpoint and its fixed dense reference and
measure the induced output-response discrepancy. Let $e_m^{\mathrm F}$ denote
the mean squared discrepancy for module $m$, aggregated over matched input
windows, perturbation directions, tokens, and fixed output sketches. We use
the normalized module share
\[
q_m^{\mathrm F}
=
\frac{e_m^{\mathrm F}}
     {\sum_{j\in\mathcal M} e_j^{\mathrm F}}
\]
and compare it with the primary backward characterization $q_{m,0}$ from
Section~\ref{sec:functional-characterization} on the same inputs.

\begin{table*}[t]
\centering
\small
\setlength{\tabcolsep}{5pt}
\renewcommand{\arraystretch}{1.06}

\caption{\textbf{Forward--backward characterization alignment.}
Module-wise agreement on matched two-window inputs at $r=0.8$. This is a separate small-input/module-level audit and is not the source data summarized in Table 1.}
\label{tab:forward_backward_alignment_full}

\begin{tabular}{lccc|ccc}
\toprule
& \multicolumn{3}{c}{\textbf{Llama-3.1-8B}}
& \multicolumn{3}{c}{\textbf{Qwen3-8B-Base}} \\
\cmidrule(lr){2-4}\cmidrule(lr){5-7}
Method
& Spearman & Cosine & Top-8
& Spearman & Cosine & Top-8 \\
\midrule
\textbf{ASVD}
& 0.940$^{*}$ & 0.971 & 5/8
& 0.662 & 0.783 & 7/8 \\
\textbf{SVD-LLM}
& 0.614 & 0.690 & 6/8
& 0.811 & 0.813 & 5/8 \\
\textbf{Basis Sharing}
& 0.738 & 0.983 & 7/8
& 0.877 & 0.822 & 6/8 \\
\textbf{MoDeGPT}
& 0.934 & 0.992 & 7/8
& 0.714 & 0.481 & 3/8 \\
\textbf{GFWSVD}
& 0.869 & 0.955 & 6/8
& 0.848 & 0.592 & 5/8 \\
\textbf{Swift-SVD}
& 0.626 & 0.739 & 6/8
& 0.825 & 0.821 & 5/8 \\
\bottomrule
\end{tabular}
\end{table*}

\paragraph{Additional checks.}
Exact share agreement is less uniform than rank agreement: Qwen MoDeGPT and
GFWSVD, for example, have cosine similarities of $0.481$ and $0.592$, showing
that the two characterizations need not assign the same mass even when their
module rankings remain positively aligned. Comparing the same forward
characterization against the independent Q128 backward characterization used
in Fig.~\ref{fig:pressure_organization} yields positive Spearman correlations
for all twelve endpoints ($0.506$--$0.829$, median $0.601$). The smaller-scale
Llama ASVD audit also preserves the qualitative response discrepancy. Together,
these checks support using the two characterizations as related but distinct
views of the same compression-induced functional structure.

\paragraph{Localized recovery: energy responses and comparison protocol}
\label{app:s6_local_reorganization_protocol}

Figure~\ref{fig:s6_local_energy} complements the relative share maps by tracking absolute pressure energy ($E_{m,t}$). Across all four localized interventions, single-module updates induce widespread absolute-energy attenuation across unupdated layers, though different intervention sites yield distinct spatial decay profiles. Crucially, this confirms that a module's relative pressure share can increase even while its absolute energy declines. For post hoc evaluation, responses are benchmarked over 60 unupdated modules common to all interventions, with attenuation significance validated against 1,999 paired bootstrap resamples and multi-seed directional consistency checks.

\begin{figure}[t]
\centering
\includegraphics[width=\linewidth]{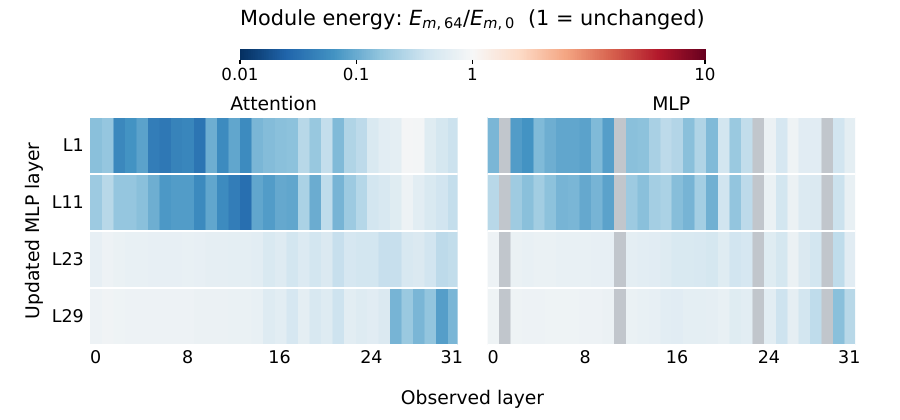}
\caption{\textbf{Absolute-energy responses to localized recovery.} Normalized energy change ($E_{m,64}/E_{m,0}$ on Q128; 3-seed mean) across 60 unupdated modules under single-MLP BASIS recovery (SVD-LLM, Llama-3.1-8B, $r=0.8$). Rows denote updated layers; blue indicates attenuation ($<1.0$) and gray marks excluded sites. Relative share growth can coexist with absolute energy decay.}
\label{fig:s6_local_energy}
\end{figure}

\subsubsection{Recovery-Profile Dynamics and Reproducibility}
\label{app:s6_module_dynamics}
The new trajectory analysis covers SVD-LLM, Swift-SVD, and MoDeGPT on
Llama-3.1-8B at retention 0.8, with C4 recovery seeds 7711--7713 and
steps 0, 64, 128, 192, 256, 320, and 384. Source and dense states are
shared across seeds. All 64 module interfaces are measured; only the
preselected layers 1, 16, and 31 appear in the main figure. No new
recovery training is performed. The normalized share is computed within
each seed before aggregation. The full-module atlas reports all layers,
with each seed retained and individual subplot y limits explicitly noted.

For each seed and checkpoint, 999 paired document bootstrap resamples
reuse the same Q128 indices for source, recovered, and dense states.
The error reference is the 95th percentile of the centered profile-change
TV errors. Each reference is a finite-document diagnostic conditional on
the observed recovery seed; it does not provide simultaneous coverage
across all comparisons. Mean-seed profile comparisons use the same
resampled documents across seeds and average the seed-specific normalized
profiles before computing TV. They do not resample recovery seeds.
The plotted seed range is descriptive and differs from these document
error references. Across the three illustrated methods, all 18
seed-averaged source-to-checkpoint profile changes exceed the paired-document
error reference. Eight of nine seed-specific source-to-step384 comparisons
do so; the exception is MoDeGPT seed 7712, with TV 0.315 below its 0.344
reference. These counts do not imply that every seed or method has a
resolved change at every checkpoint.

\begin{table}[t]
\centering\small
\caption{Full-profile change over steps 0--384 for each selected recovery seed. Ref. is the paired-document error reference. $E\downarrow$ counts absolute-energy decreases out of 64 modules; these sign counts are descriptive.}
\begin{tabular}{@{}lrrrr@{}}
\toprule Method & Seed & TV & Ref. & $E\downarrow$ \\
\midrule
SVD-LLM & 7711 & 0.479 & 0.210 & 64 \\
SVD-LLM & 7712 & 0.414 & 0.290 & 64 \\
SVD-LLM & 7713 & 0.369 & 0.311 & 64 \\
Swift-SVD & 7711 & 0.334 & 0.255 & 64 \\
Swift-SVD & 7712 & 0.436 & 0.290 & 64 \\
Swift-SVD & 7713 & 0.376 & 0.276 & 64 \\
MoDeGPT & 7711 & 0.414 & 0.290 & 61 \\
MoDeGPT & 7712 & 0.315 & 0.344 & 61 \\
MoDeGPT & 7713 & 0.408 & 0.308 & 60 \\
\bottomrule
\end{tabular}
\end{table}
\begin{table}[t]
\centering\small
\caption{All measured checkpoints: seed-mean profile TV and its paired-document error reference. Cells contain TV / Ref.}
\begin{tabular}{@{}lrrrrrr@{}}
\toprule Method & 64 & 128 & 192 & 256 & 320 & 384 \\
\midrule
SVD-LLM & 0.355/0.130 & 0.443/0.154 & 0.398/0.177 & 0.366/0.219 & 0.415/0.148 & 0.417/0.191 \\
Swift-SVD & 0.366/0.175 & 0.410/0.122 & 0.401/0.227 & 0.318/0.184 & 0.365/0.199 & 0.379/0.254 \\
MoDeGPT & 0.312/0.217 & 0.336/0.183 & 0.487/0.237 & 0.354/0.258 & 0.311/0.299 & 0.370/0.243 \\
\bottomrule
\end{tabular}
\end{table}
\paragraph{Full endpoint cohort.}
The broader endpoint cohort uses Q128, seed 7711, and steps 0 and 384
across all six methods and both backbones at retention 0.8. All twelve
source-to-final profile changes exceed their paired-document references.
These are not three-seed, seven-checkpoint results for all twelve endpoints.
\begin{table}[t]
\centering\small
\caption{Complete endpoint cohort at retention 0.8; one recovery seed.}
\begin{tabular}{@{}llrr@{}}
\toprule Backbone & Method & TV & Ref. \\
\midrule
Llama & ASVD & 0.493 & 0.271 \\
Llama & SVD-LLM & 0.479 & 0.210 \\
Llama & Basis Sharing & 0.495 & 0.293 \\
Llama & MoDeGPT & 0.414 & 0.288 \\
Llama & GFWSVD & 0.768 & 0.265 \\
Llama & Swift-SVD & 0.334 & 0.257 \\
Qwen & ASVD & 0.502 & 0.206 \\
Qwen & SVD-LLM & 0.424 & 0.239 \\
Qwen & Basis Sharing & 0.323 & 0.154 \\
Qwen & MoDeGPT & 0.540 & 0.406 \\
Qwen & GFWSVD & 0.342 & 0.199 \\
Qwen & Swift-SVD & 0.661 & 0.354 \\
\bottomrule
\end{tabular}
\end{table}
\begin{table}[t]
\centering\small
\caption{Initial module shares (percent) for the layers in the main figure. A and M denote attention and MLP. These baselines are shared across recovery seeds.}
\begin{tabular}{@{}lrrrrrr@{}}
\toprule Method & L1 A & L16 A & L31 A & L1 M & L16 M & L31 M \\
\midrule
SVD-LLM & 0.489 & 0.124 & 5.307 & 69.276 & 0.204 & 6.368 \\
Swift-SVD & 0.518 & 0.125 & 5.214 & 69.105 & 0.204 & 6.257 \\
MoDeGPT & 0.265 & 0.131 & 24.457 & 44.557 & 0.206 & 16.196 \\
\bottomrule
\end{tabular}
\end{table}

\subsubsection{Robustness Across Recovery Seeds and Loss Operating Points}
\label{app:pressure_stopping_seeds}

We further test whether the controlled comparison in
Section~\ref{sec:beyond-loss} is sensitive to the recovery trajectory or to
the particular loss operating point used by the stopping rule.

\paragraph{Across recovery seeds.}
Table~\ref{tab:pressure_stopping_seeds} reports the full seed-level results for
the four backbone--retention settings evaluated with three recovery seeds.
Each seed row averages the six compression methods, while the mean rows
aggregate all 18 method--seed trajectories within a setting.


\begin{table*}[t]
\centering
\caption{
\textbf{Seed-level controlled stopping comparison.}
Each seed row averages six compression methods; mean rows aggregate all
18 method--seed trajectories within each setting.
Full recovery uses 384 updates at $r=.6$ and 1024 at $r=.4$.
Update reduction is measured relative to \textsc{Loss};
$\Delta$Avg7 is \textsc{Loss+P} minus \textsc{Loss}.
}
\label{tab:pressure_stopping_seeds}

\scriptsize
\setlength{\tabcolsep}{3.2pt}
\renewcommand{\arraystretch}{1.08}

\resizebox{\textwidth}{!}{%
\begin{tabular}{
llc
r
>{\columncolor{pressureblue}}r
r
r r
>{\columncolor{pressureblue}}r
r
}
\toprule
&&&
\multicolumn{3}{c}{\textbf{Recovery updates}} &
\multicolumn{4}{c}{\textbf{Avg7} $\uparrow$} \\
\cmidrule(lr){4-6}\cmidrule(lr){7-10}

Backbone & $r$ & Seed
& \textsc{Loss}
& \cellcolor{pressurebluehead}\textbf{\textsc{Loss+P}}
& \shortstack{Update reduction\\(\%)}
& Full
& \textsc{Loss}
& \cellcolor{pressurebluehead}\textbf{\textsc{Loss+P}}
& $\Delta$Avg7 \\
\midrule

\multirow{4}{*}{Llama-3.1-8B}
& \multirow{4}{*}{.6}
& 7711 & 362.7 & 288.0 & 20.59 & 49.803 & 49.825 & 49.782 & $-0.043$ \\
&& 7712 & 362.7 & 288.0 & 20.59 & 49.074 & 49.241 & 49.248 & $+0.007$ \\
&& 7713 & 384.0 & 320.0 & 16.67 & 50.118 & 50.118 & 49.648 & $-0.470$ \\
&& \textbf{Mean}
& \textbf{369.8} & \textbf{298.7} & \textbf{19.23}
& \textbf{49.665} & \textbf{49.728} & \textbf{49.559} & \textbf{$-0.169$} \\

\midrule

\multirow{4}{*}{Llama-3.1-8B}
& \multirow{4}{*}{.4}
& 7711 & 960.0 & 522.7 & 45.56 & 40.640 & 40.520 & 40.516 & $-0.004$ \\
&& 7712 & 896.0 & 586.7 & 34.52 & 41.120 & 40.677 & 40.488 & $-0.189$ \\
&& 7713 & 821.3 & 576.0 & 29.87 & 40.594 & 40.588 & 40.335 & $-0.253$ \\
&& \textbf{Mean}
& \textbf{892.4} & \textbf{561.8} & \textbf{37.05}
& \textbf{40.785} & \textbf{40.595} & \textbf{40.446} & \textbf{$-0.149$} \\

\midrule

\multirow{4}{*}{Qwen3-8B-Base}
& \multirow{4}{*}{.6}
& 7711 & 373.3 & 341.3 & 8.57 & 53.194 & 53.234 & 53.156 & $-0.078$ \\
&& 7712 & 373.3 & 362.7 & 2.86 & 53.233 & 53.248 & 53.144 & $-0.105$ \\
&& 7713 & 384.0 & 352.0 & 8.33 & 53.101 & 53.101 & 53.143 & $+0.042$ \\
&& \textbf{Mean}
& \textbf{376.9} & \textbf{352.0} & \textbf{6.60}
& \textbf{53.176} & \textbf{53.194} & \textbf{53.148} & \textbf{$-0.047$} \\

\midrule

\multirow{4}{*}{Qwen3-8B-Base}
& \multirow{4}{*}{.4}
& 7711 & 906.7 & 725.3 & 20.00 & 44.725 & 44.711 & 44.753 & $+0.042$ \\
&& 7712 & 757.3 & 672.0 & 11.27 & 45.001 & 44.675 & 44.755 & $+0.080$ \\
&& 7713 & 693.3 & 821.3 & $-18.46$ & 44.834 & 43.990 & 44.771 & $+0.782$ \\
&& \textbf{Mean}
& \textbf{785.8} & \textbf{739.6} & \textbf{5.88}
& \textbf{44.853} & \textbf{44.459} & \textbf{44.760} & \textbf{$+0.301$} \\

\bottomrule
\end{tabular}%
}
\end{table*}

The aggregate behavior persists across recovery seeds. Averaged over the three seeds, \textsc{Loss+P} selects fewer updates in all four settings, with mean reductions from $5.88\%$ to $37.05\%$, while the corresponding mean Avg7 differences remain between $-0.169$ and $+0.301$ percentage points. At the individual-seed level, however, pressure does not simply shift stopping toward earlier checkpoints. In 11 of the 12 seed--setting combinations, \textsc{Loss+P} stops earlier, whereas on Qwen3-8B-Base at $r=.4$ with seed 7713 it continues beyond the loss-only stopping point ($821.3$ versus $693.3$ updates) and attains higher Avg7 ($44.771$ versus $43.990$). This variation is consistent with pressure providing trajectory-dependent information about recovery progress rather than acting as a uniformly more aggressive stopping threshold.

\paragraph{Across loss operating points.}
We next vary the loss threshold while keeping the remaining stopping protocol fixed. This tests whether the observed contribution of pressure depends on one particular operating point of the loss-only criterion. Table~\ref{tab:pressure_stopping_threshold_sweep} reports results for threshold scales $s\in\{0.5,1.0,1.5,2.0\}$.

\begin{table*}[t]
\centering
\caption{
\textbf{Controlled stopping comparison across nearby loss operating points.}
The loss threshold is scaled by $s\in\{0.5,1.0,1.5,2.0\}$ while the remaining stopping protocol is fixed. \textsc{Loss} and \textsc{Loss+P} report mean recovery updates across six compression methods. Update reduction is measured relative to \textsc{Loss}; $\Delta$Avg7 is \textsc{Loss+P} minus \textsc{Loss} in percentage points. The final column reports the number of methods whose Avg7 difference remains within $\pm1$ pp. We use seed 7711 that is also in Table~\ref{tab:downstream-main}
}
\label{tab:pressure_stopping_threshold_sweep}

\small
\setlength{\tabcolsep}{5.0pt}
\renewcommand{\arraystretch}{1.10}

\begin{tabular}{
lcc
r
>{\columncolor{pressureblue}}r
rrc
}
\toprule
Backbone & $r$ & Scale $s$
& \textsc{Loss}
& \cellcolor{pressurebluehead}\textbf{\textsc{Loss+P}}
& \shortstack{Update reduction\\(\%)}
& $\Delta$Avg7 (pp)
& Within $\pm1$ pp \\
\midrule

\multirow{8}{*}{Llama-3.1-8B}
& \multirow{4}{*}{.6}
& $0.5\times$ & 362.7 & 352.0 & 2.94  & $-0.112$ & 6/6 \\
&& $1.0\times$ & 362.7 & 288.0 & 20.59 & $-0.043$ & 6/6 \\
&& $1.5\times$ & 352.0 & 266.7 & 24.24 & $+0.047$ & 6/6 \\
&& $2.0\times$ & 341.3 & 245.3 & 28.12 & $-0.012$ & 5/6 \\
\cmidrule(lr){2-8}
& \multirow{4}{*}{.4}
& $0.5\times$ & 970.7 & 661.3 & 31.87 & $-0.003$ & 6/6 \\
&& $1.0\times$ & 960.0 & 522.7 & 45.56 & $-0.004$ & 6/6 \\
&& $1.5\times$ & 896.0 & 458.7 & 48.81 & $-0.662$ & 5/6 \\
&& $2.0\times$ & 789.3 & 405.3 & 48.65 & $-0.846$ & 4/6 \\

\midrule

\multirow{8}{*}{Qwen3-8B-Base}
& \multirow{4}{*}{.6}
& $0.5\times$ & 384.0 & 384.0 & 0.00 & $+0.000$ & 6/6 \\
&& $1.0\times$ & 373.3 & 341.3 & 8.57 & $-0.078$ & 6/6 \\
&& $1.5\times$ & 373.3 & 341.3 & 8.57 & $-0.078$ & 6/6 \\
&& $2.0\times$ & 352.0 & 341.3 & 3.03 & $-0.083$ & 6/6 \\
\cmidrule(lr){2-8}
& \multirow{4}{*}{.4}
& $0.5\times$ & 906.7 & 789.3 & 12.94 & $+0.029$ & 6/6 \\
&& $1.0\times$ & 906.7 & 725.3 & 20.00 & $+0.042$ & 6/6 \\
&& $1.5\times$ & 768.0 & 725.3 & 5.56 & $+0.016$ & 6/6 \\
&& $2.0\times$ & 682.7 & 661.3 & 3.12 & $+0.150$ & 5/6 \\

\bottomrule
\end{tabular}
\end{table*}

Across the tested operating points, adding pressure generally changes the selected recovery point while preserving similar downstream Avg7. The effect size varies with the loss threshold, as expected, but the qualitative result does not depend on the default operating point: pressure continues to affect the stopping decision across a broad range of loss thresholds, with most method-level Avg7 differences remaining within $\pm1$ percentage point. These results support the interpretation of recovery pressure as providing information about recovery progress that is complementary to the scalar loss signal, rather than reproducing one particular loss-threshold choice.

%% file: tables/full_pressure_recovery.tex
\newcommand{\best}[1]{\textbf{#1}}
\newcommand{\pos}[1]{\textcolor{teal!60!black}{#1}}
\newcommand{\nega}[1]{\textcolor{red!60!black}{#1}}

\begin{table*}[t]
\centering
\footnotesize
\setlength{\tabcolsep}{2.5pt}
\renewcommand{\arraystretch}{1.1}
\caption{\textbf{Model recoverability and endpoint pressure across low-rank baselines.} Endpoint and +Recovery report zero-shot Avg7 (\%) before and after 384-update BASIS recovery (mean $\pm$ SD across 3 seeds), with Gain showing net improvement. $P_0$ measures the initial pressure magnitude on the C4 probe. Superscripts denote rank within each setting; bold marks the best score.}
\label{tab:main_recoverability}
\begin{tabular*}{\textwidth}{@{\extracolsep{\fill}}llrrrr|rrrr@{}}
\toprule
\multirow{2}{*}{$r$} & \multirow{2}{*}{\textbf{Method}} & \multicolumn{4}{c|}{\textbf{Llama-3.1-8B}} & \multicolumn{4}{c}{\textbf{Qwen3-8B-Base}} \\
\cmidrule(lr){3-6}\cmidrule(l){7-10}
& & \textbf{Endpoint} & \textbf{+Recovery} & \textbf{Gain} & $P_{0}$ & \textbf{Endpoint} & \textbf{+Recovery} & \textbf{Gain} & $P_{0}$ \\
\midrule
\multirow{6}{*}{0.8} & ASVD & 33.55$^{(6)}$ & \cellcolor{basisblue}55.14{\tiny $\pm0.12$}$^{(6)}$ & \pos{+21.60} & 1.544 & \textbf{59.61}$^{(1)}$ & \cellcolor{basisblue}\textbf{62.40}{\tiny $\pm0.62$}$^{(1)}$ & \pos{+2.78} & 0.296 \\
 & SVD-LLM & 51.94$^{(2)}$ & \cellcolor{basisblue}58.80{\tiny $\pm0.03$}$^{(3)}$ & \pos{+6.86} & 0.699 & 55.93$^{(3)}$ & \cellcolor{basisblue}60.57{\tiny $\pm0.92$}$^{(4)}$ & \pos{+4.64} & 0.398 \\
 & Basis Sharing & 51.17$^{(4)}$ & \cellcolor{basisblue}57.97{\tiny $\pm0.09$}$^{(4)}$ & \pos{+6.79} & 0.868 & 53.97$^{(6)}$ & \cellcolor{basisblue}60.42{\tiny $\pm0.49$}$^{(5)}$ & \pos{+6.45} & 0.509 \\
 & MoDeGPT & \textbf{61.10}$^{(1)}$ & \cellcolor{basisblue}\textbf{62.45}{\tiny $\pm0.04$}$^{(1)}$ & \pos{+1.35} & 0.662 & 54.48$^{(5)}$ & \cellcolor{basisblue}59.35{\tiny $\pm0.35$}$^{(6)}$ & \pos{+4.86} & 0.448 \\
 & GFWSVD & 48.34$^{(5)}$ & \cellcolor{basisblue}57.25{\tiny $\pm0.23$}$^{(5)}$ & \pos{+8.91} & 1.163 & 55.25$^{(4)}$ & \cellcolor{basisblue}60.64{\tiny $\pm0.11$}$^{(3)}$ & \pos{+5.39} & 0.609 \\
 & Swift-SVD & 51.55$^{(3)}$ & \cellcolor{basisblue}58.91{\tiny $\pm0.04$}$^{(2)}$ & \pos{+7.35} & 0.705 & 56.10$^{(2)}$ & \cellcolor{basisblue}61.09{\tiny $\pm0.87$}$^{(2)}$ & \pos{+5.00} & 0.409 \\
\midrule
\multirow{6}{*}{0.6} & ASVD & 32.55$^{(6)}$ & \cellcolor{basisblue}44.86{\tiny $\pm2.40$}$^{(6)}$ & \pos{+12.31} & 2.178 & 34.60$^{(6)}$ & \cellcolor{basisblue}52.35{\tiny $\pm0.12$}$^{(5)}$ & \pos{+17.75} & 1.615 \\
 & SVD-LLM & 36.81$^{(2)}$ & \cellcolor{basisblue}50.92{\tiny $\pm0.42$}$^{(3)}$ & \pos{+14.11} & 0.836 & 43.12$^{(3)}$ & \cellcolor{basisblue}54.57{\tiny $\pm0.26$}$^{(2)}$ & \pos{+11.45} & 0.505 \\
 & Basis Sharing & 33.67$^{(4)}$ & \cellcolor{basisblue}48.47{\tiny $\pm0.99$}$^{(5)}$ & \pos{+14.80} & 1.572 & 44.03$^{(2)}$ & \cellcolor{basisblue}53.23{\tiny $\pm0.35$}$^{(4)}$ & \pos{+9.20} & 0.920 \\
 & MoDeGPT & \textbf{45.94}$^{(1)}$ & \cellcolor{basisblue}\textbf{53.80}{\tiny $\pm0.68$}$^{(1)}$ & \pos{+7.87} & 0.645 & 41.50$^{(5)}$ & \cellcolor{basisblue}49.77{\tiny $\pm0.12$}$^{(6)}$ & \pos{+8.27} & 0.575 \\
 & GFWSVD & 32.98$^{(5)}$ & \cellcolor{basisblue}48.78{\tiny $\pm0.22$}$^{(4)}$ & \pos{+15.79} & 3.614 & \textbf{45.77}$^{(1)}$ & \cellcolor{basisblue}53.82{\tiny $\pm0.23$}$^{(3)}$ & \pos{+8.05} & 1.410 \\
 & Swift-SVD & 36.39$^{(3)}$ & \cellcolor{basisblue}51.16{\tiny $\pm0.74$}$^{(2)}$ & \pos{+14.77} & 0.835 & 43.04$^{(4)}$ & \cellcolor{basisblue}\textbf{55.31}{\tiny $\pm0.12$}$^{(1)}$ & \pos{+12.27} & 0.504 \\
\bottomrule
\end{tabular*}
\end{table*}

%% file: tables/main_numeric_geometry_shift.tex
\begin{table*}[ht]
\centering
\small
\setlength{\tabcolsep}{4.5pt}
\renewcommand{\arraystretch}{1.08}

\caption{
\textbf{Controlled fixed-coordinate projection diagnostic for factor-update recovery} (LLaMA-3.1-8B, \(r=0.4\)). Normal evaluates each method in its learned coordinates; Projected maps recovered back to Source-defined fixed-coordinate subspace. 
}

\label{tab:geometry_shift_main}
\resizebox{\textwidth}{!}{
\begin{tabular}{lrr|rrrr|rr}
\toprule
\multirow{2}{*}{\textbf{Recovery method}} &
\multicolumn{2}{c|}{\textbf{Endpoint Avg7$\uparrow$}} &
\multicolumn{4}{c|}{\textbf{Avg7 projection diagnostic}} &
\multicolumn{2}{c}{\textbf{PPL Avg$\downarrow$}} \\
\cmidrule(lr){2-3}
\cmidrule(lr){4-7}
\cmidrule(l){8-9}
&
\textbf{Learned} &
\textbf{Projected} &
\textbf{Total} &
\textbf{Projected gain} &
\textbf{Coord.-shift} &
\textbf{Shift share} &
\textbf{Learned} &
\textbf{Projected} \\
\midrule
 \rowcolor{basisblue} BASIS
& \textbf{50.48} & \textbf{50.48}
& \pos{\textbf{+17.29}} & \pos{\textbf{+17.29}}
& \textbf{+0.00} & \textbf{0.0\%}
& \textbf{16.57} & \textbf{16.57} \\
U-LoRA
& 50.15 & 39.17
& \pos{+16.96} & \pos{+5.98}
& +10.98 & 64.7\%
& 17.34 & 223.96 \\
V-LoRA
& 49.91 & 41.80
& \pos{+16.73} & \pos{+8.61}
& +8.12 & 48.5\%
& 17.90 & 119.25 \\
JFT
& 50.26 & 40.85
& \pos{+17.07} & \pos{+7.67}
& +9.40 & 55.1\%
& 17.56 & 283.05 \\
\bottomrule
\end{tabular}
}
\end{table*}